\documentclass[pdflatex,sn-mathphys-ay]{sn-jnl}

\usepackage{graphicx}%
\usepackage{multirow}%
\usepackage{amsmath,amssymb,amsfonts}%
\usepackage{amsthm}%
\usepackage{mathrsfs}%
\usepackage[title]{appendix}%
\usepackage{xcolor}%
\usepackage{textcomp}%
\usepackage{manyfoot}%
\usepackage{booktabs}%
\usepackage{algorithm}%
\usepackage{algorithmicx}%
\usepackage{algpseudocode}%
\usepackage{listings}%
\usepackage[utf8]{inputenc}
\usepackage[T1]{fontenc}

\usepackage{mathtools, bm, bbm}
\usepackage{geometry}
\usepackage{graphicx, subcaption, xcolor, array, tikz-cd}

\usepackage{hyperref}
\usepackage{braket}

\usepackage[capitalize]{cleveref}

\theoremstyle{thmstyleone}%
\newtheorem{thm}{Theorem}
\newtheorem{prop}{Proposition}
\newtheorem{lemma}{Lemma}

\theoremstyle{thmstyletwo}%
\newtheorem{rem}{Remark}

\theoremstyle{thmstylethree}%
\newtheorem{definition}{Definition}%

\graphicspath{{figures/}}

\input{defs.tex}

\makeatletter

\def\acknowname{Acknowledgements}%
\def\acknow#1{\ifx#1\empty\else\def\@acknow{\par\addvspace{10pt}{\keywordfont{\bfseries\acknowname:} #1\par}}\fi}%
\def\@acknow{}%

\def\keywordname{Keywords}%
\def\keywords#1{\ifx#1\empty\else\def\@keywords{\par\addvspace{10pt}{\keywordfont{\bfseries\keywordname:} #1\par}}\fi}%
\def\@keywords{}%

\def\printkeywords{\ifx\@keywords\empty\else\@keywords\fi\par%
    \ifx\@acknow\empty\else\@acknow\fi\par}
\makeatother

\begin{document}
\title{Revisiting the Sliced Wasserstein Kernel for Persistence Diagrams: a Figalli--Gigli approach.}


\author[1]{\fnm{Marc} \sur{Janthial} \email{marc.janthial@polytechnique.org}}

\author[2]{\fnm{Théo} \sur{Lacombe} \email{theo.lacombe@univ-eiffel.fr}}

\affil[1]{\orgname{École polytechnique}, \orgaddress{\city{Palaiseau}, \country{France}}}

\affil[2]{\orgdiv{Laboratoire d’Informatique Gaspard Monge}, \orgname{Univ. Gustave Eiffel, CNRS, LIGM, F-77454}, \orgaddress{\city{Marne-la-Vallée}, \country{France}}}


\abstract{ 
    The Sliced Wasserstein Kernel (SWK) for persistence diagrams was introduced in \citep{pmlr-v70-carriere17a} as a powerful tool to implicitly embed persistence diagrams in a Hilbert space with reasonable distortion. 
    This kernel is built on the intuition that the Figalli--Gigli distance---that is the partial matching distance routinely used to compare persistence diagrams---resembles the Wasserstein distance used in the optimal transport literature, and that the later could be sliced to define a positive definite kernel on the space of persistence diagrams. This efficient construction nonetheless relies on ad-hoc tweaks on the Wasserstein distance to account for the peculiar geometry of the space of persistence diagrams. 
    
    In this work, we propose to revisit this idea by directly using the Figalli--Gigli distance instead of the Wasserstein one as the building block of our kernel. On the theoretical side, our sliced Figalli--Gigli kernel (SFGK) shares most of the important properties of the SWK of Carrière et al., including distortion results on the induced embedding and its ease of computation, while being more faithful to the natural geometry of persistence diagrams. In particular, it can be directly used to handle infinite persistence diagrams and persistence measures. On the numerical side, we show that the SFGK performs as well as the SWK on benchmark applications.
}


\acknow{Authors thank Clément Bonet for fruitful discussions. The work of TL is supported by the Agence Nationale de la Recherche (ANR) under grant ANR-24-CE23-7711 (project TheATRE). This work was done when MJ was intern at Laboratoire d’Informatique Gaspard Monge (Univ. Gustave Eiffel, CNRS).   
}



\maketitle

\section{Introduction}
Topological Data Analysis (TDA) is an emerging field in data analysis which aims to design topological descriptors of complex structured objects. 
Its main tools are built upon persistent homology theory, and the most well-known descriptor it produces is called the persistence diagram (PD). 
PDs enjoy strong stability properties with respect to perturbation of the data \citep{stability_Chazal}, and have found many applications in several fields such as computer graphics \citep{pascucci2010topological,carriere2015stable,tierny2017topology}, material science \citep{hiraoka2016hierarchical,saadatfar2017pore,olejniczak2023topological}, computational biology \citep{bukkuri2021applications,aukerman2022persistent,chung2024morphological}, to name a few.  
Nevertheless, their use in those applications is not straightforward. 
Indeed, PDs take the form of point clouds with multiplicities in $\bR^2$ and are typically compared using transport-like metrics \citep{FG10, DL20}, which are quite expensive to compute in practice \citep{peyré2020computationaloptimaltransport}. 
Furthermore, the space of PDs equipped with such metrics is not Hilbert \citep{turner_frechet_2014, turner_same_2020, bubenik_embeddings_2020}, preventing their direct use in learning methods which require that structure on the descriptor space (e.g.~PCA, SVM). 
A workaround explored in the literature consists in defining kernels on that space to map PDs to vectors in a (possibly infinite-dimensional) Hilbert space. 
Several contributions have been made using this approach \citep{reininghaus-pss, JMLR:v18:17-317, JMLR:v16:bubenik15a} by defining an explicit embedding of PDs in a Hilbert space. 

Particularly inspiring for this work, \citet{pmlr-v70-carriere17a} proposed the \textit{Sliced Wasserstein Kernel} (SWK) based on the \textit{Sliced Wasserstein distance} used in computational optimal transport \citep{rabin2011wasserstein}. 
That distance enjoys important stability properties, is fairly easy to compute and---of crucial importance---is provably conditionally negative definite on the space of finite PDs and can therefore be used to define a (Gaussian or Laplace) kernel on the space of persistence diagrams. 
Their construction is presented in more details in \cref{subsubsec:kernel-for-PD}. 

\bmhead*{Contributions} In this work, we introduce a new approach to defining sliced distances on the space of persistence diagrams that remains faithful to its underlying geometry. 
Building on the framework of \citet{DL20}, which relies on the Figalli–Gigli metric~\citep{FG10}—a variant of the Wasserstein distance specifically designed to account for the distinguished role of the diagonal—we define the \emph{Sliced Figalli–Gigli distance} between persistence diagrams.
A key feature of the resulting distance is that, unlike its Sliced Wasserstein counterpart, it naturally extends to infinite persistence diagrams and, more generally, to persistence measures. This places $\SFG$ within a broader and more intrinsic geometric framework for comparing topological descriptors, rather than restricting it to the finite-diagram setting.
We show that this increased generality does not come at the expense of either theoretical guarantees or computational tractability. In particular, we establish stability results for $\SFG$ comparable\footnote{and, to some extent, more general as we derive them for any exponent $p \geq 1$ while \citet{pmlr-v70-carriere17a} restrict the analysis to $p=1$.} to those known for the Sliced Wasserstein distance, and we propose efficient algorithms for its computation. Our numerical experiments further demonstrate that kernels derived from $\SFG$ achieve empirical performance on par with that of the Sliced Wasserstein kernel. 
Overall, the $\SFG$ distance and its associated kernel can be viewed as a principled refinement of the Sliced Wasserstein kernel of~\citet{pmlr-v70-carriere17a}, providing a unified framework that applies to a wider class of persistence-based representations while preserving both theoretical soundness and practical efficiency.

\bmhead*{Outline.} In \cref{sec:background}, we introduce the necessary background on persistence diagrams and kernel methods before reviewing related works in the literature. 
In \cref{sec:SFG} we define the \textit{Sliced Figalli--Gigli} distance and study its theoretical properties. 
In \cref{sec:expe}, we give algorithms for efficient computation of that new distance and discuss the experimental performances of the resulting kernel in numerical applications. 

\section{Background}\label{sec:background}
\subsection{Persistent homology and persistence diagrams} \label{sect:pers-hom}
Persistent homology is a machinery based on algebraic topology used to define stable descriptors of real-valued functions on topological spaces. 
Persistence diagrams (PDs) are one of the most well-known of these descriptors and they, roughly speaking, encode information about \textit{topological components} (such as loops, connected components, enclosed surfaces...) of the underlying function and topological space (see \citep{edelsbrunner2010computational} for an introduction).

More precisely, given a topological space $X$ and a function $f: X \to \bR$, define $H_k(t)$ to be the $k$-th homology group of $X_t \coloneqq f^{-1}([-\infty, t])$ over an arbitrary field. 
For $s \leq t$, the inclusion $X_s \hookrightarrow X_t$ induces a map $\iota_{s,t}$ between the homology groups $H_k(s)$ and $H_k(t)$. By functoriality of homology, these maps satisfy that for all $r \leq s \leq t$ the following diagram commutes:

\begin{equation}\label{dgrm:pers-homology}
\begin{tikzcd}
H_k(r) \arrow[rr, "\iota_{r,t}"] \arrow[dr, "\iota_{r,s}", swap] & & H_k(t) \\
& H_k(s) \arrow[ur, "\iota_{s,t}"'] &
\end{tikzcd}
\end{equation}
The collection of the homology groups $(H_k(t))_{t\in\bR}$ along with the maps $(\iota_{s,t})_{s,t \in \bR}$ satisfying \eqref{dgrm:pers-homology} form a \textit{persistence module} $M_f$. 
It has been shown \citep{botnan2019decompositionpersistencemodules} that under general assumptions this persistence module can be uniquely decomposed over a family $(k_I)_{I\subset \bR}$ indexed over the intervals of $\bR$. 
Given the decomposition $M_f \coloneqq \bigoplus_j k_{I_j}$, the \textbf{persistence diagram} $D$ of $f$ is the multiset of the endpoints of the intervals $I_j$ and is therefore a multiset supported on the open half-plane $\groundspace \coloneqq \{(b,d) \in \bR^2, b < d\}$. 
Following \citet{chazal:hal-01330678}, one may equivalently think of PDs as \textbf{locally finite measures supported on $\groundspace$ of the form $\sum_xn_x\delta_x$}, where $\delta_x$ denotes the Dirac mass located at $x \in \groundspace$ and $n_x \in \bN$ denotes the multiplicity of $x$. 
We adopt this perspective in the rest of this work.

A point $(b,d)$ being in $D$ can be interpreted as a $k$-dimensional homology component appearing (being "born") in $X_b$ and disappearing ("dying") in $X_d$. 
As such the persistence diagram encodes the birth and death of \textit{topological components} (connected component, loop, void \dots) at all scales. 
One of the main advantages of PDs is their stability with respect to perturbation of the data \citep{stability_Chazal} which has motivated their use in several machine learning tasks.
The interested reader may refer to \citep{edelsbrunner2010computational,oudot_persistence} for more comprehensive introductions. 

\subsection{Distances between persistence diagrams}

Treating PDs as locally finite measures supported on $\groundspace$, distances between PDs can be built following \citep{DL20}, relying on the formalism introduced by \citet{FG10}. 
Let $\cM(\groundspace)$ be the space of non-negative Radon measures supported on $\groundspace$, that is the space of non-negative measures $\mu$ supported on $\groundspace$ such that for every compact subset $A \subset \groundspace$, one has $\mu(A) < +\infty$. 
Given $\mu \in \cM(\groundspace)$, define the $p$-persistence of $\mu$ as
\begin{equation} 
\mathrm{Pers}_p(\mu) \coloneqq \left(\int_\groundspace d(x,\thediag)^p \mu(x) \right)^{\frac{1}{p}}, 
\end{equation}
where $d(x,\thediag)$ denote the Euclidean distance between a point $x \in \groundspace$ and its orthogonal projection onto the diagonal $\thediag \coloneqq \{b = d \}$. 
Let the set of measures with finite persistence, called the space of \emph{persistence measures}, be defined as
\begin{equation} \cM^p(\groundspace) \coloneqq \{\mu \in\cM(\groundspace) : \mathrm{Pers}_p(\mu) < +\infty\}. 
\end{equation}

With this formalism, the set of persistence diagrams $\cD(\groundspace)$ is the subset of $\cM(\groundspace)$ consisting of \emph{point measures}, i.e.~locally finite Radon measures of the form $\sum_{i \in I} \delta_{x_i}$ where $\delta_{x_i}$ denote the Dirac mass at $x_i \in \groundspace$, and $I$ is a (possibly countably infinite) set of indices. 
Let also $\cD^p(\groundspace) \coloneqq \cD(\groundspace) \cap \cM^p(\groundspace)$. 
The Figalli--Gigli distance $\FG_p$ is then defined on $\cM^p(\groundspace)$ as:

\begin{equation}  \tag{$\FG_p$} \label{eq:def-FG} 
    \FG_p(\mu, \nu) \coloneqq \left( \inf_{\gamma \in \Adm(\mu,\nu)} \iint_{\overline{\groundspace} \times \overline{\groundspace}}d(x,y)^p \dd \gamma(x,y)\right)^{\frac{1}{p}}.
\end{equation}

where $d$ is again the Euclidean distance\footnote{Any other $q$-norm could be used seamlessly for $q > 1$.} on $\bR^2$ and the set of admissible transport plans $\Adm(\mu,\nu)$ is the set of Radon measures supported on $\overline{\groundspace} \times \overline{\groundspace}$ satisfying that for all Borel subsets $A, B \subset \groundspace$ ;
\begin{equation} \gamma(A \times \overline{\groundspace}) = \mu(A) \quad \text{ and } \quad \gamma(\overline{\groundspace} \times B) = \nu(B).
\end{equation}

\begin{rem}\label{rem:comparison}
This formalism slightly differs from the usual definitions of PD metrics found in the TDA literature, where PDs are encoded as multi-sets on $\groundspace$, and for which \eqref{eq:def-FG} is replaced by
    \begin{equation} \tag{$d_p$}
    \label{eq:def-std-pd-dist} 
        d_p(\mu,\nu) \coloneqq \left(\inf_{\gamma \in \Gamma(\mu,\nu)} \sum_{x \in \mspt(\mu) \cup \thediag} d(x, \gamma(x))^p\right)^{\frac{1}{p}},
    \end{equation}
    where $\Gamma(\mu,\nu)$ is the set of bijections between $\mspt(\mu) \cup \thediag$ and $\mspt(\nu) \cup \thediag$, where $\mspt$ denote the support of a point measure where points are counted with multiplicity. 
    \citet{DL20} prove that \eqref{eq:def-std-pd-dist} and $\FG_p(\mu,\nu)$ coincide when $\mu,\nu \in \cD^p(\groundspace)$, the latter presenting the advantage of being well-defined on persistence measures, a more general class of objects which proved to be useful in TDA through various works \citep{chazal2018density,divol2021estimation,cao2022approximating,wu2024estimation}.
\end{rem}

\begin{rem}
These distances resemble the Wasserstein distance between probability measures introduced in the optimal transport literature (see \citep{santambrogio}), defined as 
\begin{equation} \label{eq:def-Wasserstein-distance} \tag{$\rW_p$}
\rW_p(\mu,\nu) = \left(\inf_{\gamma \in \Pi(\mu,\nu)}\iint_{\groundspace \times \groundspace}d(x,y)^p\dd \gamma(x,y)\right)^{\frac{1}{p}},
\end{equation}
where the set of admissible transport plans $\Pi(\mu,\nu)$ is in this case the set of measures supported on $\groundspace \times \groundspace$ with marginals $\mu$ and $\nu$.

This similarity led the TDA community to refer to \eqref{eq:def-std-pd-dist} as the ``Wasserstein distance between persistence diagrams'' \citep{turner_frechet_2014,berwald2018computing,skraba2020wasserstein,bubenik2022virtual}. 
We stress that \eqref{eq:def-std-pd-dist} nonetheless differs from \eqref{eq:def-Wasserstein-distance}: the Wasserstein distance is a distance between non-negative measures with the \emph{same (finite) total mass}, where the whole mass of a source measure $\mu$ must be transported exactly to the target measure $\nu$. 
In particular, in contrast to \eqref{eq:def-std-pd-dist}, \eqref{eq:def-Wasserstein-distance} does not allow one to transport an arbitrary (possibly infinite) amount of mass to and from the diagonal $\thediag$. 
The Sliced Wasserstein Kernel of Carrière et al.~\citep{pmlr-v70-carriere17a} is built on the actual Wasserstein distance \eqref{eq:def-Wasserstein-distance} in dimension one, while the kernel we introduce in this work relies on \eqref{eq:def-FG}. 
To avoid confusion, we will refer to \eqref{eq:def-FG} (and \eqref{eq:def-std-pd-dist}) as the Figalli--Gigli distance between persistence measures (in particular diagrams).
\end{rem}

\subsection{Kernel methods} \label{subsec:kernel-methods}

\bmhead*{Positive definite kernels.}
Given a set $X$ and a symmetric function $k \colon X \times X \to \bR$, $k$ is said to be a positive semi-definite kernel if for every input $(x_1, \dots, x_n) \in X^n$ the matrix $(k(x_i, x_j))_{ij} \in \bR^{n \times n}$ is itself positive semi-definite. 
Positive semi-definite kernels (simply referred to as kernels for the sake of concision) are useful as they are equivalent to feature maps: for every kernel $k$ there exists a Hilbert space $\cH$, called the Reproducing Kernel Hilbert Space (RKHS), and a feature map $\phi : X \to \cH$ for which the following holds:
\begin{equation} 
\forall x, y \in X, \quad \langle\phi(x), \phi(y)\rangle_\cH = k(x, y).
\end{equation}
As such, for all downstream learning tasks relying only on values of the inner product between elements of $\phi(X)$ (including many learning methods such as PCA, SVM, Ridge regression, etc.) only the value of $k(x,y)$ is necessary and no explicit knowledge of the embeddings $\phi(x),\phi(y)$ is required. 
This idea is known as the ``kernel trick''.

\bmhead*{Conditionally negative definite kernels.} A common way of defining kernels is, given a Hilbert space with distance $d$, to set $k_\sigma(x,y) \coloneqq \exp(-\frac{d(x,y)^2}{\sigma^2})$ for $\sigma > 0$. 
More generally, \citep[Theorem 3.2.2]{Berg1984} states that setting $k_\sigma(x,y) \coloneqq \exp(-\frac{f(x,y)}{\sigma^2})$ only yields a kernel if the function $f : X \times X \to \bR$ is conditionally negative definite, that is:
\begin{equation}
    \forall x_1, \dots, x_n \in X, \forall a_1, \dots, a_n \in \bR, \quad \sum_{i=1}^n a_i = 0\implies \sum_{i,j} a_ia_jf(x_i,x_j) \leq 0.
\end{equation}
As shown in \citep{curvature}, if the square of the distance function of a metric space $(X,d)$ is conditionally negative definite, then $X$ is flat (i.e.~isometrically embeddable in a Hilbert space) or $\mathrm{CAT}(0)$. 
However since $\cD^p$ is not $\mathrm{CAT}(k)$ for any $k > 0$ \citep{turner_frechet_2014}, $d_p^2$ is not c.n.d for any $p$. 
This motivates the search for other ways to define kernels on PDs that would nonetheless be faithful to the geometry induced by the Figalli--Gigli distance. 

\subsection{Related works} \label{subs:related-works}

\subsubsection{The Sliced Wasserstein distance for probability measures}

While the Wasserstein distance between probability measures \eqref{eq:def-Wasserstein-distance} is expensive to compute (typically of order $O(N^3)$ for two measures supported on $N$ points), the problem is known to become much simpler in dimension $1$: in that case, the optimal transport plan $\gamma$ is supported on the graph of the (unique) monotone map matching the quantiles of $\mu$ to that of $\nu$, i.e.~$x \mapsto F_\nu^{-1} \circ F_\mu(x)$, where $F_\rho(x) \coloneqq \rho((-\infty,x])$ for a probability distribution $\rho$ supported on $\bR$. 
See \citep[Ch.~2]{santambrogio} for an extensive overview. 
In particular, if $\mu$ (resp.~$\nu$) is supported on $N$ points $x_1 < \dots < x_N$ (resp.~$y_1 < \dots < y_N$) with uniform weight $\frac{1}{N}$, one simply has $\rW_p(\mu,\nu)^p = \sum_{i=1}^n |x_i - y_i|^p$ ; reducing the computational cost to $O(N \log(N))$ (that of sorting the points). 

Building on this idea, \citet{rabin2011wasserstein} introduce the \emph{Sliced Wasserstein distance} between probability measures:
\begin{equation} \label{eq:def-sliced-W-proba}
    \mathrm{SW}_p(\mu,\nu)^p \coloneqq \int_{\theta \in S^{d-1}} W_p(\pi_\theta \# \mu, \pi_\theta \# \nu)^p \dd \theta,
\end{equation}
where $\pi_\theta \colon \bR^d \to \bR$ is the projection $x \mapsto \braket{x,\theta}$, and $\#$ denotes the pushforward operator, i.e.~$\pi_\theta \# \mu(A) = \mu(\pi_\theta^{-1}(A))$ for all Borel set $A \subset \bR$. 
Approximating the integral by sampling $K$ direction on the sphere $S^{d-1}$, one obtains a Wasserstein-like distance between probability measures that can be computed in $O(K N \log(N))$, a significant improvement over \eqref{eq:def-Wasserstein-distance}. 

The Sliced Wasserstein distance has attracted attention on the theoretical side \citep{nadjahi2020statistical,xi2022distributional,park2025geometry} and has found several applications in machine learning \citep{deshpande2018generative,nadjahi2021sliced}. 
In particular \citet{kolouri2016sliced} observed that $(\mu,\nu) \mapsto \exp \left( - \frac{\mathrm{SW}_2(\mu,\nu)^2 }{2 \sigma^2} \right)$ define a Gaussian-like kernel on the set of probability measures. 

Several extensions have been proposed, for instance by replacing the average by a maximum in \eqref{eq:def-sliced-W-proba} \citep{deshpande2019max, kolouri2019generalized,boedihardjo2025sharp}, and adapting it non-Euclidean geometries \citep{bonet2023sliced, bonet2023hyperbolic, bonet2025sliced}, where projections on lines are replaced by (Busemann) projection on geodesics. 
We will take inspiration of this framework\footnote{Though that literature focuses on negatively curved manifolds, which unfortunately does not cover our framework.} when defining the Sliced Figalli--Gigli distance between persistence diagrams. 

\subsubsection{Kernels for Persistence Diagrams}\label{subsubsec:kernel-for-PD}
Several kernels on persistence diagrams have been proposed in the literature, typically divided in two classes. 
The first consists of defining kernels using explicit feature maps: given any map $\phi \colon \cD \to \bR^d$, one can build by definition a kernel by setting $(\mu,\nu) \mapsto \braket{\phi(\mu),\phi(\nu)}$. This includes for instance \citep{JMLR:v16:bubenik15a,adams2017persistence}.
Closer to our purpose, the second class consists of defining implicit embeddings by directly building kernels on $\cD$. 
For the sake of completeness, we briefly review the first two kernels of that type, before presenting in details the Sliced Wasserstein Kernel for persistence diagrams, from which our work is vastly inspired.

\bmhead*{The persistence ``Weighted Gaussian'' (PWG) and ``Scale-Space'' (PSS) kernels.} 
These two kernels are the first examples of PSD kernels on PDs defined through the kernel operator $k \colon \cD^p(\groundspace) \times \cD^p(\groundspace) \to \bR$, instead of relying on an explicit embedding of PDs in a Hilbert space. 
The PWG kernel is introduced in \citep{pmlr-v48-kusano16, JMLR:v18:17-317} and is defined in the following way. 
Let $C > 0$, $\mu$ be a persistence diagram, $k_\rho$ be the gaussian kernel on $\groundspace \times \groundspace$ with parameter $\sigma > 0$ and $\cH_\sigma$ the corresponding RKHS.
Define $\alpha_\mu \coloneqq \sum_{x \in \mspt(\mu)} \arctan(Cd(x,\thediag)^p)k_\sigma(\cdot, x) \in \cH_\sigma$ to be the kernel mean embedding of $\mu$ weighted by the distance of its points to the diagonal.
Eventually, define for $\tau > 0$
\begin{equation} k_\mathrm{PWG}(\mu_1, \mu_2) \coloneqq \exp\left(-\frac{\|\alpha_{\mu_1} - \alpha_{\mu_2}\|^2}{\tau^2}\right).
\end{equation}
The weighting of the embedding is motivated by the intuition that generators far from the diagonal carry more information than those that are close to it which typically reflect noise in the data.

The PSS kernel is defined in \citep{reininghaus-pss} as the scalar product in $L^2(\overline{\groundspace})$ of the two solutions of the heat diffusion equations with Dirac sources located at every point of the diagrams, namely:
\begin{equation}
    k_\mathrm{PSS}(\mu_1, \mu_2) \coloneqq \frac{1}{8\pi\sigma}\sum_{\substack{p \in \mspt(\mu_1 )\\ q \in \mspt(\mu_2)}}\exp\left(-\frac{\|p-q\|^2}{8\sigma} \right) - \exp\left(-\frac{\|p - \bar q\|^2}{8\sigma}\right),
\end{equation}
where $\bar q$ is the symmetric of $q$ with respect to the diagonal, and $\sigma > 0$. 

\bmhead*{The Sliced Wasserstein Kernel (SWK).} 

While PWG and PSS kernels can be proved to be stable, few is known about there discriminative power. 
To improve on this, \citet{pmlr-v70-carriere17a} introduced a \emph{Sliced Wasserstein kernel for persistence diagrams}, proposing an adaptation of the Sliced Wasserstein distance defined in \eqref{eq:def-sliced-W-proba} to the context of persistence diagrams.

More precisely, given two persistence diagrams $\mu, \nu$ and $\theta \in \bS$, let $\mu_\theta \coloneqq \pi_\theta\#(\mu + \pi_\Delta\#\nu)$ and $\nu_\theta \coloneqq \pi_\theta\#(\nu + \pi_\Delta\#\mu)$ where $\pi_\theta$ is the orthogonal projection on the line $l_\theta$ going through $(0,0)$ with angle $\theta$ to the diagonal and $\pi_\Delta$ is the orthogonal projection on the diagonal. This corresponds to projecting all the points of one diagram on the diagonal and then projecting those points along with the ones from the other diagram on $l_\theta$. The kernel is then defined as:
\begin{equation}
k_\mathrm{SW}(\mu, \nu) \coloneqq \exp\left(-\frac{\SW(\mu, \nu)}{\sigma^2}\right),
\end{equation}
where $\SW(\mu, \nu) = \int_\bS \rW_1(\mu_\theta, \nu_\theta) \dd \theta$. 
This does yield a valid kernel since the Wasserstein distance on 1-dimensional measures is the $L^1$ distance between their quantile functions and is therefore c.n.d. 
This special form of the Wasserstein distance in one dimension also allows for efficient computation of this kernel as computing the quantile function of a 1-dimensional empirical measure boils down to sorting its points.

The authors also obtain bounds relating the $\SW$ distance and the usual $d_1$ distance for finite persistence diagrams. 
Namely,
\begin{equation}\label{eq:result-carriere-stability}
    \frac{d_1(\mu,\nu)}{1 + 2N(N-1)} \leq \SW(\mu,\nu) \leq 2\sqrt2 d_1(\mu,\nu)
\end{equation}
for any persistence diagrams $\mu,\nu \in \cD(\groundspace)$ with less than $N$ points. 
This kernel proved to perform better in various tasks compared to the PSS and PWG kernels while also being more computationally efficient.

\begin{rem}
The $\SW$ distance is only defined for finite persistence diagrams and it is unclear whether the construction proposed in \citep{pmlr-v70-carriere17a} can be generalized to infinite persistence diagrams or persistence measures. 
In particular, projected measures $\mu_\theta$ may not be Radon and as such the Wasserstein distance between them would be undefined. 
The formalism we introduce in this work, while substantially similar, presents the advantage of being faithfully defined for arbitrary persistence diagrams and measures. 
Furthermore, the bounds derived by \citet{pmlr-v70-carriere17a} only consider the exponent $p=1$, while the ones we present in this work hold for any $p \geq 1$. 
\end{rem}

\section{The Sliced Figalli--Gigli distance and the induced kernel}\label{sec:SFG}
Following the formalism introduced in \citep{DL20}, we propose to adopt a similar approach as in \citep{pmlr-v70-carriere17a} by \textit{slicing} the $\FG_p$ distance using \emph{geodesics} emanating from the diagonal $\thediag$.

\subsection{Motivation \& Defintion}
Consider $x,y \in \groundspace$. Observe that as $x$ and $y$ approach $\thediag$, one has $\FG_p(\delta_x,\delta_y) \to 0$.
Therefore, equipping $\cM^p(\groundspace)$ with the distance $\FG_p$ invites us to identify all points of the diagonal $\thediag$. 
More formally, $\FG_p$ induces a distance function $\Delta_p$ on $\overline{\groundspace}$ defined by
\begin{equation}
\Delta_p^p(x,y) \coloneqq \FG_p^p(\delta_x, \delta_y) = \min(d(x,y)^p, d(x, \pi(x))^p + d(y,\pi(y))^p),
\end{equation}
where $\pi$ is the orthogonal projection on the diagonal.
This distance is related to the natural distance induced by $d$ on $\tilde\groundspace \coloneqq \groundspace \cup \{\thediag\}$ the quotient of the closed half-plane $\overline{\groundspace}$ by $\thediag$. Indeed, $d$ induces a function $\tilde d$ on $\tilde\groundspace$ which satisfies $\tilde d(x,\thediag) = d(x, \pi(x))$. We then define a distance $\rho$ on $\tilde \Omega$ by setting:
\begin{equation}
    \rho(x, y) = \min(\tilde{d}(x,y), \tilde{d}(x, \partial\Omega) + \tilde{d}(y, \partial\Omega)).
\end{equation}
As such, $\tilde\Omega$ seems to better represent the geometry of the space of persistence diagrams endowed with $\FG_p$. 
Following ideas introduced in \citep{bonet2023leveraging}, it would then seem natural to define a \textit{Sliced Figalli--Gigli} metric by projecting on geodesics of $\tilde\groundspace$ passing through a given origin $O$. 
Since the point $\thediag$ plays a particular role in the structure of $\tilde\groundspace$ we consider geodesics passing through this point. 
We have the following straightforward result.
\begin{prop}[Geodesics in $\tilde\groundspace$]
    Let $x, y \in \tilde\groundspace$,
    \begin{enumerate}
    \item If $\rho(x,y) = \tilde{d}(x, y)$ the geodesic from $x$ to $y$ is the straight line from $x$ to $y$, i.e.,
    \begin{equation}
        \gamma_{x,y}(t) = (1-t)x + ty.
    \end{equation}
    \item If $\rho(x,y) = \tilde{d}(x, \partial\Omega) + \tilde{d}(y, \partial\Omega) = \|x - \pi(x)\| + \|y - \pi(y)\|$, the geodesic from $x$ to $y$ goes through the diagonal and we have
    \begin{equation}
        \gamma_{x,y} \colon
        \left\{
        \begin{array}{ll}
             t \in [0, \alpha_x] \mapsto (1 - \frac{t}{\alpha_x})x + \frac{t}{\alpha_x}\thediag \\[8pt]
             t \in[\alpha_x, 1] \mapsto (1 - \frac{t - \alpha_x}{\alpha_y})y + \frac{1 - t}{\alpha_y}\thediag
        \end{array}
        \right.
    \end{equation}
    where $\alpha_x = \frac{\|x - \pi(x)\|}{\|x - \pi(x)\| + \|y - \pi(y)\|}$ and $\alpha_y = 1 - \alpha_x$.
\end{enumerate}
\end{prop}

As such, the geodesics originating from $\thediag$ in $\tilde\groundspace$ are of the form $G_t \coloneqq \{(t,t) + s(-1,1), s > 0\} \cup \{\thediag\}$ (which we call the geodesic originating from $\thediag$ with parameter $t$) for $t\in \bR$. 
We then have the following result.
\begin{prop}[Projections on geodesics]
    Let $t \in \bR$. The projection on $G_t$ is given by:
    \begin{equation} \label{eq:pi-t}
    \forall z \in \tilde\groundspace, \quad
    \pi_t(z) =
    \begin{cases}
        (t,t) + \frac{z_2-z_1}{2}(-1,1) &\text{if } z_1\leq t \leq z_2 \\
        \thediag &\text{otherwise.}
    \end{cases}
    \end{equation}
\end{prop}

\begin{proof}
    Let $t \in \bR$, $G_t$ be the geodesic originating from $\thediag$ with parameter $t$ and $\cG_t$ its pullback in $\bR^2$ i.e $\cG_t \coloneqq \{(t,t) + s (-1,1), s\geq 0\}$. The projection of $z = (z_1, z_2) \in \groundspace$ onto $\cG_t$ is then $\Pi_t(z) \coloneqq (t,t) + \frac{z_2- z_1}{2}(-1, 1)$. We then have
\begin{equation}
\|z-\Pi_t(z)\|^2 = 2\left(\frac{z_1+z_2}{2} - t \right)^2 \ \mathrm{and} \ \|z - \pi(z)\|^2 = \frac{(z_2 -z_1)^2}{2}.
\end{equation}
Therefore,
\begin{equation} 
    \|z-\Pi_t(z)\|^2 = \|z-\pi(z)\|^2 \iff (z_1-t)^2 + (z_2 - t)^2 = (z_2 - z_1)^2 \iff t = z_1 \text{ or }\ t = z_2.
\end{equation}
The equation for $\pi_t$ follows.
\end{proof}

\begin{figure}[!h]
        \centering
        \includegraphics[width = 0.49\linewidth]{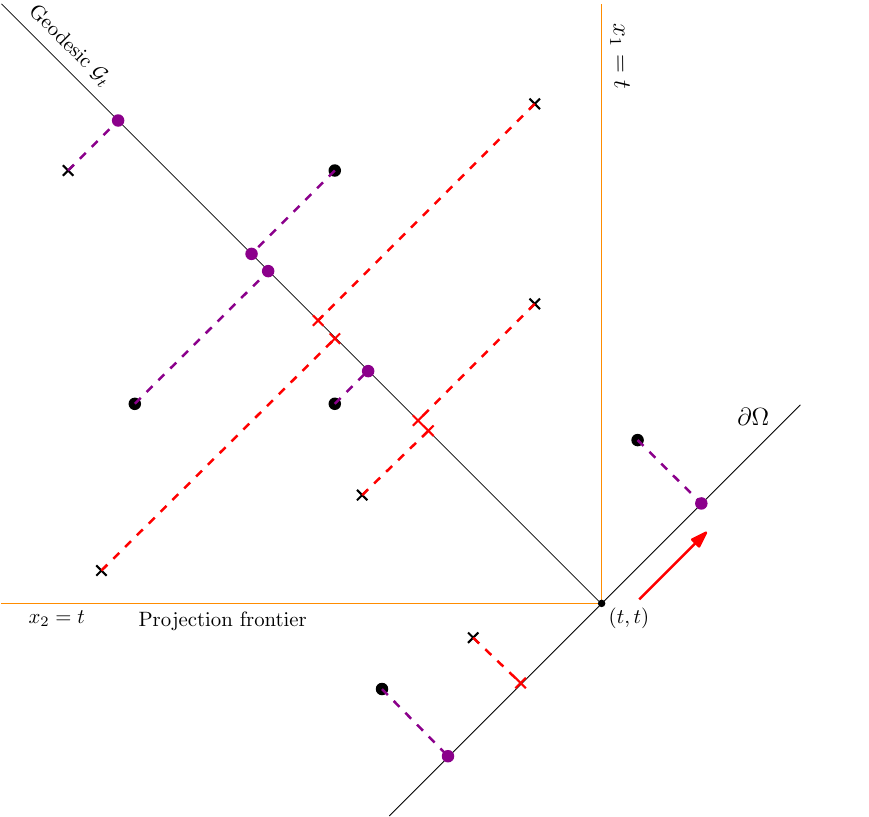}
            \caption{Projections of two measures $\mu$ (circles) and $\nu$ (crosses)  and their respective projections on the geodesic with parameter $t$. Points outside of the region delimited by $x_1 = t$ and $x_2 = t$ (which we denote $\groundspace_t$) are projected onto $\thediag$. As $t$ describes $\bR$, $\groundspace_t$ \textit{slides} across the upper half-plane and captures information about different parts of the diagram. Observe also that a point $x$ belongs to $\groundspace_t$ for $t$ belonging to an interval of length proportional to $d(x,\thediag)$, hence the need for the renormalization we introduce.}
        \label{fig:proj-n-SFG}
\end{figure}

The following lemma shows that the projection onto geodesics preserves the Radon structure of the measures.
\begin{lemma}
    Let $\mu \in \cM(\groundspace)$ and $t \in \bR$. Then for any compact $A \subset G_t \setminus \{\thediag\}$, $\pi_t\#\mu(A) < + \infty$.
\end{lemma}
\begin{proof}
Let $M \coloneqq \sup(\{d(x,\thediag), x \in A\}$. Observe that since $A \subset G_t\setminus\thediag$, we have $\pi_t^{-1}(A) \subset C \coloneqq \{z \in \groundspace, t-\sqrt2 M \leq z_1 \leq t \text{ and }  t \leq z_2 \leq t + \sqrt2 M \}$ which is a compact subset of $\groundspace$. Hence, $\pi_t\#\mu(A) = \mu(\pi_t^{-1}(A)) < +\infty$.
\end{proof}
\begin{rem}
    The above Lemma shows that the pushforward of any Radon measure $\mu \in \cM(\groundspace)$ by the projection $\pi_t$ is still a Radon measure when restricting it to $\cG_t \setminus \thediag$. The fact that it is not a Radon measure on $\cG_t$ (since the point $\thediag$ may have infinite mass) won't come into question as we will only consider the restriction of the pushforward measure to $\cG_t \setminus \thediag$ later on (which we still denote by $\pi_t \# \mu$).
\end{rem}

One may then want to consider the quantity $
    \left(\int_{t\in\bR} \FG_p^p(\pi_t\#\mu, \pi_t\#\nu)\dd t\right)^{\frac{1}{p}} $ to define a distance on $\cM^p(\groundspace)$.
However, given $\mu \in \cM(\groundspace)$, we have
\begin{align}
       \int_{t\in\bR}\FG_p^p(\pi_t\#\mu, \emptyset)\dd t
        &= \int_{t\in\bR}\left(\int_{\Omega}d(\pi_t(z), \thediag)^p\dd \mu(z)\right)\dd t \\
        &=\int_{\Omega}(z_2-z_1)d(z, \thediag)^p\dd \mu(z) \\
        &=\sqrt{2}\mathrm{Pers}_{p+1}(\mu),
\end{align}
and therefore the above quantity is only well-defined on $\cM^{p+1}(\groundspace)$ and not on $\cM^p(\groundspace)$. 
To circumvent this, we introduce the following renormalization. 
Given $\mu \in \cM^p(\groundspace)$ define $\tilde\mu$ by setting for any Borel set $A$
\begin{equation} \label{eq:def-normalized-measure}
\tilde\mu(A) \coloneqq \int_A \frac{\dd\mu(x)}{d(x,\thediag)}.
\end{equation}
This way, $\tilde\mu \in \cM^{p+1}(\Omega)$ since $\tilde\mu(A) < +\infty$ for all compact sets of $\groundspace$ (every compact of the open half-plane $\groundspace$ is at a positive distance from $\thediag$) and $\mathrm{Pers}_{p+1}(\tilde\mu) = \mathrm{Pers}_p(\mu) < +\infty$. 
This yields the following definition.

\begin{definition} [The Sliced Figalli--Gigli distance] \label{def:SFG-n}
    Let $p \geq 1$ and $\mu, \nu \in \cM^p(\groundspace)$ and $\tilde\mu,\tilde\nu$ their normalization following \eqref{eq:def-normalized-measure}. 
    The Sliced Figalli--Gigli (SFG) distance between $\mu$ and $\nu$ is defined as
    \begin{equation}  \label{eq:def-SFG} \tag{$\SFG_p$}
    \SFG_p(\mu,\nu) \coloneqq \frac{1}{\sqrt2}\left(\int_\bR \FG_p^p(\pi_t\# \tilde\mu, \pi_t \# \tilde\nu)\dd t\right)^{\frac{1}{p}}.
    \end{equation}
\end{definition}

Note that given $\mu \in \cM^p(\groundspace)$ we have
\begin{align} \label{eq:sfg-empty}
     \SFG_p^p(\mu, \emptyset) &= \frac{1}{\sqrt2}\int_{t\in\bR}\left(\int_{\Omega}d(\pi_t(z), \thediag)^p\frac{\dd\mu(z)}{d(z, \thediag)}\right)\dd t \\
     & = \frac{1}{\sqrt2}\int_{\Omega}(z_2-z_1)d(z, \thediag)^{p-1}\dd \mu(z) \\
     & = \mathrm{Pers}_{p}(\mu),
\end{align}
    and therefore by Hölder's inequality,
    \begin{align}
        \SFG_p(\mu, \nu) & \leq \frac{1}{\sqrt2}\left(\int_{t\in\bR}(\FG_p(\pi_t\#\mu, \emptyset) + \FG_p(\pi_t\#\nu, \emptyset))^p\dd t\right)^\frac{1}{p} \\
        &\leq \SFG_p(\mu, \emptyset) + \SFG_p(\nu, \emptyset) < +\infty,
    \end{align}
    ensuring that \eqref{eq:def-SFG} is well defined for persistence diagrams and measures in $\cM^p(\groundspace)$. 

Equation \eqref{eq:sfg-empty} immediately gives the following proposition.

\begin{prop} \label{prop:sfg-pers}
    Let $\mu \in \cM^p(\groundspace)$, and $\emptyset$ denote the empty diagram. One has
    \begin{equation} 
    \SFG_p(\mu, \emptyset) = \FG_p(\mu,\emptyset) = \mathrm{Pers}_p(\mu).
    \end{equation}
\end{prop}

Eventually, we have the following central proposition. 

\begin{prop}\label{prop:SFG-dist}
    $\SFG_p$ is a distance on $\cM^p(\groundspace)$.
\end{prop}
\begin{proof}
    The fact that $\FG_p$ is a distance immediately implies that $\SFG_p$ is non-negative and satisfies the triangle inequality. 
    Eventually, let $\mu, \nu \in \cM^p(\groundspace)$ such that $\SFG_p(\mu,\nu) = 0$. 
    This implies that their projections (after normalization) on $\cG_t \setminus \{\thediag\}$ coincide for almost every $t \in \bR$. 
    The injectivity of that transform is proved in the following lemma, yielding $\mu = \nu$. 
\end{proof}

\begin{lemma}\label{lem:injectivity-transform}
    Let $\mu, \nu \in \cM^p(\groundspace)$ and let $(\tilde\mu_t)_{t \in \bR}, (\tilde\nu_t)_{t \in \bR}$ be their respective transform by the map $\mu \mapsto \pi_t \# \mu\restriction_{\cG_t \setminus \{\thediag\}}$. 
    Assume that $t$-a.e.~(with respect to the Lebesgue measure on $\bR$), one has $\tilde\mu_t = \tilde\nu_t$. 
    Then $\mu = \nu$. 
\end{lemma}

\begin{proof}
    To alleviate notation in this proof, we consider the change of coordinates $(x_1,x_2) \mapsto \left( \frac{y-x}{2}, \frac{x+y}{2} \right)\eqqcolon (u,v)$, and keep the same notation ($\mu,\tilde\mu_t, \groundspace, \thediag$, etc.) with this new coordinate system, see \cref{fig:new-coordinate} for an illustration. 

    \begin{figure}
        \centering
        \includegraphics[width=0.6\linewidth]{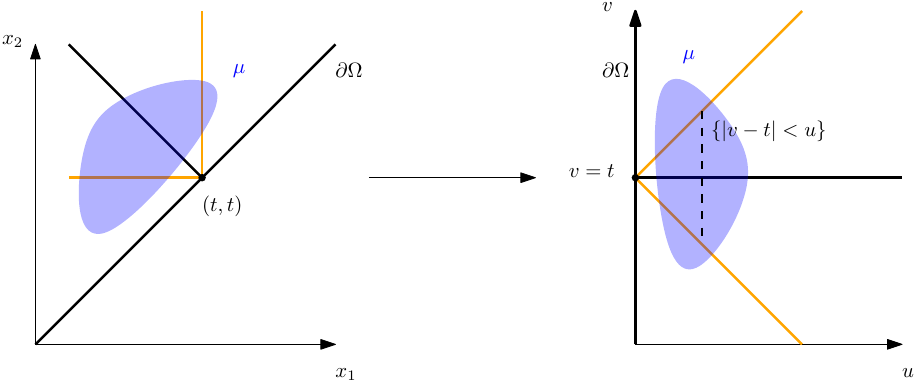}
        \caption{Change of coordinates used in the proof of \cref{lem:injectivity-transform}.}
        \label{fig:new-coordinate}
    \end{figure}

    Let $\varphi \in \cC^\infty_c((0,+\infty))$ be a smooth, compactly supported, test function. 
    One has
    \begin{equation}
        \braket{\varphi,\tilde\mu_t} = \int_{(u,v) \in \groundspace} \varphi(u) \mathbf{1}_{(v-u,v+u)}(t) \dd \mu(u,v),
    \end{equation}
    where $\mathbf{1}_{(v-u, v+u)}$ is the indicator function of the set $\{|v-t| < u\}$. 
    The first order derivative of that function with respect to $t$ in a distributional sense is $t \mapsto \int \varphi(u) [\delta_{v-u}(t) - \delta_{v+u}(t)] \dd \mu$. Therefore, considering another test function $\psi \in \cC^\infty_c(\bR)$, one has
    \begin{equation}
        \int \psi'(t) \braket{\varphi,\tilde\mu_t} \dd t = \int \varphi(u) [\psi(v-u) - \psi(v+u)] \dd \mu(u,v).
    \end{equation}
    Therefore, the equality for almost every $t \in \bR$ between $\tilde\mu_t$ and $\tilde\nu_t$ implies that for any $\varphi,\psi \in \cC^\infty_c((0,+\infty)) \times \cC^\infty_c(\bR)$, one has 
    \begin{equation}\label{eq:interm-step-lemma}
    \int \varphi(u) [\psi(v-u) - \psi(v+u)] \ \dd (\mu-\nu)(u,v) = 0.
    \end{equation}
    Let $\eta \coloneqq \mu - \nu$, let $\pi$ be the first marginal of $|\eta|(u,v)$ and consider the disintegration $\dd \eta = \dd \eta_u \otimes \dd \pi(u)$. 
    Note that the condition $\mu,\nu \in \cM^p(\groundspace)$ imposes that $\eta_u$ has finite total mass on $\bR$ for $\pi$-a.e.~$u > 0$. 
    From \eqref{eq:interm-step-lemma}, denoting $\tau_u$ the translation $w \mapsto w + u$, we have for all test function $\varphi,\psi \in \cC^\infty_c((0,+\infty)) \times \cC^\infty_c(\bR)$, 
    \begin{equation}
        \int_{u > 0} \varphi(u) \left[ \int \psi(w) \dd \tau_{-u} \# \eta_u(w) - \int \psi(w)  \dd \tau_u \# \eta_u(w) \right] \dd \pi(u) = 0.
    \end{equation}
    Therefore, we deduce that for $\pi$-a.e.~$u > 0$, one has $\int \psi(w) \dd [\tau_{-u}\# \eta_u - \tau_u \# \eta_u](w) = 0$ for all test function $\psi$, and thus that $\eta_u = \tau_{2u} \eta_u$. 
    Eventually, $\eta_u$ being of finite total mass, this translation invariance imposes $\eta_u = 0$ for $\pi$-a.e.~$u > 0$, thus $\eta = 0$, that is $\mu = \nu$. 
\end{proof}

\begin{rem}[On the choice of the projection] \label{rem:choice-of-projection}
The projection on the geodesic $\cG_t$ has no need to be the orthogonal projection. 
In particular, one may choose another map $z \mapsto \tilde\pi_t(z)$ that has nicer properties than the one defined above. 
For this reason, we study in \cref{appendix:continuous_proj} a \textit{continuous projection} which, contrary to the orthogonal projection, is Lipschitz continuous. 
One may hope that this increased regularity would lead to better theoretical guarantees, but our study did not showcase any significant improvement of this alternative projection over the naive one we present here. 
We still include this discussion in the appendix for the sake of completeness.
\end{rem}

\subsection{The Sliced Figalli--Gigli kernel}
In this section, we define the kernel associated to the $\SFG_p$ distance. 
As mentioned in \cref{subsec:kernel-methods}, we first need to prove that the $\SFG_p$ distance is c.n.d.~which, similarly to what is done in \citep{pmlr-v70-carriere17a}, will be a consequence of the special form of the $\FG_p$ distance in one dimension.

\subsubsection{The \(\FG_p\) problem in 1D}

The main result of this subsection, \cref{prop:FG-1D}, is to show that in dimension one, for Radon measures supported on the open half-line $\bR_{> 0} = (0,+\infty)$, the Figalli--Gigli distance admits a simple closed form which, as its Wasserstein counterpart \eqref{eq:def-Wasserstein-distance}, boils down to considering monotone matchings (i.e.~sorting the point for discrete measures). 
The subtlety is to properly account for the role played by the diagonal $\thediag$ and the fact that, in contrast to probability measures, persistence measures and diagram may have infinite total mass, preventing us from straightforwardly use quantile functions to define monotone matching. 

Let $\mu, \nu \in \cM^p(\bR_{> 0})$, we want to characterize $\gamma_p$ the optimal transport plan reaching the infimum in $\FG_p$ (so here $\groundspace = \bR_{> 0}$ et $\thediag = \{0\}$).
\begin{lemma} \label{lem:cycl-mon}
    Suppose $p > 1$, for all $(x_1, y_1), (x_2,y_2) \in \supp(\gamma_p)$, if $x_1 < x_2$ then $y_1 \leq y_2$.
\end{lemma}
\begin{proof}
    From \citep[Prop.~2.3]{FG10}, $\supp(\gamma_p)$ is $\Delta_p^p$-cyclically monotone, and therefore, for all $(x_1, y_1), (x_2, y_2) \in \supp(\gamma_p)$ we have
\begin{equation} |x_1 - y_1|^p + |x_2 - y_2|^p \leq |x_1 - y_2|^p + |x_2 - y_1|^p.
\end{equation}
Assume that $x_1 < x_2$, then if $y_1 > y_2$ one has that $y_1 - x_1 > y_1 - x_2 > y_2 - x_2$ and $y_1 - x_1 > y_2 - x_1 > y_2-x_2$. By strict convexity of $x \mapsto x^p$ for $p > 1$, the following then holds:
\begin{align}
    &\frac{|y_1 - x_2|^p - |y_2 - x_2|^p}{y_1 - y_2} < \frac{|y_1 - x_1|^p - |y_2 - x_2|^p}{(y_1 - x_1) - (y_2 - x_2)} < \frac{|y_1 - x_1|^p - |y_1 - x_2|^p}{x_2 - x_1} \\
    &\frac{|y_2 - x_1|^p - |y_2 - x_2|^p}{x_2 - x_1} < \frac{|y_1 - x_1|^p - |y_2 - x_2|^p}{(y_1 - x_1) - (y_2 - x_2)} < \frac{|y_1 - x_1|^p - |y_2 - x_1|^p}{y_2 - y_1}
\end{align}
From there we get that $|y_2 - x_1|^p + |y_1 - x_2|^p < |y_1 - x_1|^p + |y_2 - x_2|^p$ which is absurd. 
As such, $y_1 \leq y_2$.
\end{proof}

We will now prove, similarly to how it is done it the case of traditional Optimal Transport \citep{santambrogio}, that $\FG_p$ in dimension 1 is actually an $L^p$ distance.
\begin{definition}
    Let $\mu \in \cM(\bR_+)$, for $x \in \bR_+$ we define:
    \begin{equation} 
    F_\mu(x) = \mu([x, +\infty)) \text{ and } F_\mu^{-1}(x) = \sup(\{t\in\bR_+ : F_\mu(t) \geq x\}),
    \end{equation}
    agreeing that $F_\mu^{-1}(x) = 0$ if $\{t\in\bR_+ : F_\mu(t) \geq x\} = \emptyset$.
\end{definition}
\begin{rem}
    These definitions are very close to the standard definitions for the distribution function of a measure and its pseudo-inverse. The main difference with our constructions (which we will refer to with the same names) is the "direction" of the quantiles: we define them as $\mu([x, + \infty) )$ instead of the standard $\mu( (-\infty, x])$. 
    The latter could be ill-defined if $\mu$ has infinite total mass close to $\thediag$, making the alternate definition necessary. 
\end{rem}

\begin{rem}
    $F_\mu$ characterises $\mu$ since the family $\{[x, +\infty) , x\in\bR_+\}$ generates all the Borel sets of $\bR_+$. 
    Furthermore, the following holds:
\begin{equation} \label{eq:FG-1D-lem}
    F_\mu^{-1}(x) \geq a \iff F_\mu(a) \geq x.
\end{equation}
\end{rem}

\begin{prop}
    Let $\mu,\nu \in \cM(\bR_+)$. 
    One has
    \begin{equation} 
    F_\mu^{-1} \# \cL^1([0, +\infty)) = \mu.
    \end{equation}
    Furthermore, denote $\gamma_{mon} = (F_\mu^{-1}, F_\nu^{-1}) \# \cL^1([0, +\infty))$, then $\gamma_{mon} \in \Adm(\mu, \nu)$ and $\gamma_{mon}([a, +\infty)\times[b, +\infty)) = \min(F_\mu(a), F_\nu(b))$.
\end{prop}

\begin{proof}
    We have 
    \begin{align}
        F_{F_\mu^{-1} \# \cL^1([0, +\infty))}(x) &= \cL^1(\{x\geq 0, F_\mu^{-1}(x) \geq a\}) \\
        &= \cL^1(\{x \geq 0 : F_\mu(a) \geq x\}) \tag{see \ref{eq:FG-1D-lem}}\\
        &= F_\mu(a).
    \end{align}
    As such, these two measures have the same distribution function so $F_\mu^{-1}\#\cL^1([0, +\infty)) = \mu$. Then:
    \begin{align}
        \iint_{A\times\bR_+}\dd \gamma_{mon}(x,y) &= \int_{\bR_+}1_A(F_\mu^{-1}(x))\, 1_{\bR_+}F_\nu^{-1}(x)\dd x \\
        &= \int_A\dd (F_\mu^{-1}\#\dd x)(x) \\
        &= \mu(A). \tag{according to the first part}
    \end{align}
    Similarly, one can show that the second marginal of $\gamma_{mon}$ is $\nu$, so $\gamma_{mon} \in \Adm(\mu, \nu)$. 
    Then,
    \begin{align}
        \gamma_{mon}([a, +\infty), [b, +\infty)) &= \cL^1(\{x \geq 0 : F_\mu^{-1}(x) \geq a, F_\nu^{-1}(x) \geq b\}) \\
        &= \cL^1(\{x \geq 0 : F_\mu(a) \geq x, F_\nu(b) \geq x\}) \\
        &= \min(F_\mu(a), F_\nu(b)).
    \end{align}
    which is the desired equality.
\end{proof}

\begin{prop}
    $\gamma_{mon}$ is optimal.
\end{prop}
\begin{proof}
    We still suppose that $p> 1$. Let $\gamma_p$ denote the optimal transport plan, then from \ref{lem:cycl-mon} we get that for all $(x,y), (x', y') \in \supp(\gamma_p), x<x' \implies y\leq y'$. Consider $a, b \in \bR_+$ and let $A = [a, +\infty) \times [0, b)$ and $B = [0, a )\times [b, +\infty)$. If $\gamma_p(A) > 0$ and $\gamma_p(B) > 0$, then we can find $(x,y) \in A \cap \supp(\gamma_p)$ and $(x',y') \in B \cap \supp(\gamma_p)$. But then we would have $x' < x $ and $y' > y$ which is absurd. \\
    Hence,
    \begin{align}
        \gamma([a, +\infty)\times [b, +\infty)) = \min( \gamma([a, +\infty)\times [b, +\infty) \cup A),  \gamma([a, +\infty)\times [b, +\infty) \cup B)).
    \end{align}
    However, $\gamma([a, +\infty)\times [b, +\infty) \cup A) = \gamma([a, +\infty) \times \bR_+) = F_\mu(a)$ and $\gamma([a, +\infty)\times [b, +\infty) \cup B) = \gamma(\bR_+ \times [b, +\infty)) = F_\nu(b)$. Therefore, $\gamma_p$ and $\gamma_{mon}$ coincide on sets of the form $[a, +\infty)\times [b, +\infty)$ which is enough to conclude that $\gamma_p = \gamma_{mon}$.\\
    The case for $p=1$ is handled with the same limit argument as in \citep{santambrogio}. Lemma 2.10 in \citep{santambrogio} states that for every $\varepsilon > 0$ there exist a stricly convex function $c_\varepsilon$ such that $|\cdot| \leq c_\varepsilon \leq \varepsilon + (1+\varepsilon)|\cdot| $. We then get, given $\gamma \in \Adm(\mu,\nu)$:
    \begin{align}
        \int_{\bR+}|x-y| \dd \gamma_{mon}(x,y) &\leq \int_{\bR_+}c_\varepsilon(x-y)\dd \gamma_{mon}(x,y) \\
        & \leq \int_{\bR_+}c_\varepsilon(x-y)\dd \gamma(x,y) \tag{$\gamma_{mon}$ is optimal since $c_\varepsilon$ is strictly convex} \\
        & \leq \varepsilon + \int_{\bR_+}(1+\varepsilon)|x-y|\dd \gamma(x,y).
    \end{align}
    Taking the limit as $\varepsilon \to 0$ then yields the desired result.
\end{proof}

From this we deduce the following result.
\begin{prop} \label{prop:FG-1D}
    Let $\mu,\nu \in \cM^p(\bR_{> 0})$,
\begin{equation} 
\FG_p^p(\mu, \nu) = \int_{\bR_{> 0}}|F_\mu^{-1}(t) - F_\nu^{-1}(t)|^p\dd t.
\end{equation}
\end{prop}
This is analogous to the special form of the standard Wasserstein distance in 1D.

\subsubsection{The $\SFG$ kernel}
Before defining the $\SFG$ kernel we first need to prove that $\SFG$ is CND.

\begin{prop}
    For all $1 \leq p \leq 2$, $\SFG_p^p$ is CND.  
\end{prop}

\begin{proof}
    Let $1 \leq p \leq 2$, $\mu_1, \dots \mu_n \in \cM^p(\groundspace)$ and $a_1, \dots, a_n \in \bR$ such that $\sum_ia_i =0$. 
    We have
        \begin{align}
            \sum_{i,j}a_ia_j\SFG_p^p(\mu_i, \mu_j) &= \frac{1}{\sqrt2^p}\sum_{i,j}a_ia_j\int_\bR \FG_p^p(\pi_t \# \tilde\mu_i, \pi_t\# \tilde\mu_j)\dd t \\
            &= \frac{1}{\sqrt2^p}\sum_{i,j}a_ia_j\int_\bR\int_{\bR_+} |F_{\pi_t\#\tilde\mu_i}^{-1}(s) - F_{\pi_t\#\tilde\mu_j}^{-1}(s)|^p\dd s \dd t \\
            &= \frac{1}{\sqrt2^p}\int_\bR \int_{\bR_+}\left(\sum_{i,j} a_ia_j |F_{\pi_t\#\tilde\mu_i}^{-1}(s) - F_{\pi_t\#\tilde\mu_j}^{-1}(s)|^p\right)\dd s \dd t.
        \end{align}
     We denote by $d$ the euclidean distance on $\bR$. Since $d^2$ is CND, Theorem 4.7 in \citep{embeddings_wellswilliams} states that $d^p$ also is. 
     As such, 
     \begin{equation} 
     \forall t \in \bR, \quad \sum_{i,j}a_ia_j |F_{\pi_t\#\tilde\mu_i}^{-1}(s) - F_{\pi_t\#\tilde\mu_j}^{-1}(s)|^p \leq 0.
     \end{equation}
     Thus we obtain $\sum_{i,j}a_ia_j\SFG_p(\mu_i, \mu_j) \leq 0$ i.e. $\SFG_p^p$ is CND. 
\end{proof}

Hence, \citep[Theorem 3.2.2]{Berg1984} allows us to define a valid kernel on $\cM^p(\groundspace)$ through
\begin{equation} 
k_{\SFG_p}(\mu,\nu) \coloneqq \exp\left(-\frac{\SFG_p^p(\mu,\nu)}{\sigma^2}\right).
\end{equation}
Note that this kernel is defined for $p \in [1,2]$; for $p=1$ one has a Laplace-like kernel on persistence measures while for $p=2$ this kernel resembles a Gaussian kernel. 

\subsection{Stability and equivalence of topologies} \label{sec:stability}
In this subsection, we prove the stability of the $\SFG_p$ distance with respect to the $\FG_p$ distance as well as the equivalence of the topologies induced by these norms.

\subsubsection{Stability}
We want to prove that the $\SFG_p$ distance is stable with respect to the $\FG_p$ distance, i.e.~obtain an inequality of the form $\SFG_p(\mu,\nu) \leq C_p\FG_p(\mu,\nu)$ for some constant $C_p$. 
We prove the following result.

\begin{thm} \label{thm:ineg-projn}
    Let $\mu, \nu \in \cM^p(\groundspace) \cap \cM^\infty(\groundspace)$, one has:
    \begin{equation}
    \SFG_p^p(\mu,\nu) \leq \FG_p^p(\mu,\nu) + 2M^{p-1} \iint_{\groundspace\times\groundspace}d(x,y)\dd\gamma(x,y),
    \end{equation}
    where $M \coloneqq \max(\mathrm{Pers}_\infty(\mu), \mathrm{Pers}_\infty(\nu))$. 
    For $p=1$, the result simplifies to
    \begin{equation}
    \forall\mu, \nu \in \cM^1(\groundspace), \:{\SFG}_1(\mu,\nu) \leq 3\FG_1(\mu,\nu),
    \end{equation}
    in particular one does not need to assume that $\max(\mathrm{Pers}_\infty(\mu), \mathrm{Pers}_\infty(\nu)) < \infty$.
\end{thm}

Establishing the stability of the Sliced Wasserstein distance with respect to the Wasserstein distance derived in \citep{pmlr-v70-carriere17a} when $p=1$ (see \eqref{eq:result-carriere-stability}) relies on the idea that one may obtain a transport plan between the projection of two measures $\mu, \nu$ from the transport plan $\gamma$ between the original measures by simply considering the pushforward of $\gamma$ by the projection. 
Our proof relies on the same idea though the construction of an admissible transport plan between the projected measures is more technical as we are renormalizing our measures before projecting them on geodesics. \\

Let $\mu,\nu \in \cM^p(\groundspace)$ and $\gamma \in \Adm(\mu, \nu)$, we can construct an element $\tilde\gamma$ of $\Adm(\tilde\mu, \tilde\nu)$ from $\gamma$ by setting for $A, B \subset \groundspace$:
\begin{equation}
    \tilde\gamma(A\times B) \coloneqq \iint_{A\times B}\frac{1}{\max(d(x, \thediag), d(y, \thediag))}\dd \gamma(x,y).
\end{equation}
Furthermore, let $A_+ = \{(x,y) \in A \times\groundspace : d(y, \thediag) \geq d(x, \thediag)\}$ and $B_+ = \{(x,y) \in \groundspace\times B : d(x, \thediag) \geq d(y, \thediag)\}$ and define:
\begin{align}
    &\tilde\gamma(A \times \thediag) \coloneqq \iint_{A \times \thediag}\frac{1}{d(x,\thediag)}\dd \gamma(x,y) + \iint_{A_+}\left(\frac{1}{d(x, \thediag)} - \frac{1}{d(y, \thediag)}\right)\dd \gamma(x,y) \\
  & \tilde\gamma(\thediag \times B) \coloneqq \iint_{\thediag \times B}\frac{1}{d(y, \thediag)}\dd \gamma(x,y) + \iint_{B_+}\left(\frac{1}{d(y, \thediag)} - \frac{1}{d(x, \thediag)}\right) \dd \gamma(x,y).
\end{align}

\begin{prop} \label{prop:tilde-transport-plan}
    With the above definition, $\tilde\gamma \in \Adm(\tilde\mu,\tilde\nu)$.
\end{prop}

\begin{proof}
    We have
\begin{align}
    \tilde\gamma(\overline{\groundspace} \times B) &= \iint_{\thediag \times B}\dd \tilde\gamma(x,y) + \iint_{B_+}\dd \tilde \gamma(x,y) + \iint_{B_-}\dd \tilde \gamma(x,y) \\
    &=\iint_{\thediag \times B}\frac{\dd \gamma(x,y)}{d(y, \thediag)} + \iint_{B_+}\left(\frac{1}{d(y, \thediag)} - \frac{1}{d(x, \thediag)}\right)\dd \gamma(x,y) \\
    & \quad \quad + \iint_{B_+}\frac{1}{d(x, \thediag)}\dd \gamma(x,y) + \iint_{B_-}\frac{1}{d(y, \thediag)}\dd \gamma(x,y) \\
    &= \iint_{\overline{\groundspace}\times B}\frac{1}{d(y, \thediag)}\dd \gamma(x,y) = \tilde\nu(B).
\end{align}
Similarly, we have $\tilde\gamma(A \times \overline{\groundspace}) = \tilde\mu(A)$. \\

We now just need to verify that $\tilde\gamma$ takes finite values on compacts of $\groundspace\times\groundspace \setminus \thediag\times\thediag$ (cf. \citep{FG10}, transport plans are Radon measures on that space). Consider then $C$ a compact of that space. Observe that, since $\thediag\times\thediag$ is closed, $\varepsilon \coloneqq d(C, \thediag\times\thediag) > 0$. We can then write $C \subset \bigcup_{(x,y) \in C} B(x, \frac{\varepsilon}{4})\times B(y, \frac{\varepsilon}{4})$ from which we extract a finite subcover to get $C \subset \bigcup_{i=1}^n\overline{B(x_i, \frac{\varepsilon}{4})}\times \overline{B(y_i,,\frac{\varepsilon}{4})}$. Therefore, it suffices to show that $\gamma(A\times B)$ is finite for all $A \times B$ compact de $\groundspace\times\groundspace$ such that $d(A\times B, \thediag \times \thediag) > 0$.

Let $A, B$ be such compacts. Since $d(A\times B, \thediag\times\thediag) > 0)$, necessarily $d(A, \thediag) > 0$ or $d(B, \thediag) > 0$. Suppose it is the case for $A$. We then write $A \times B \subset A\times(B\setminus\thediag) \cup A\times \thediag$ from which we get $\tilde\gamma(A\times B) \leq \tilde\gamma(A\times B\setminus \thediag) + \tilde\gamma(A\times\thediag)$. Furthermore,
\begin{align}
    \tilde\gamma(A\times (B\setminus\thediag)) &= \iint_{A\times B\setminus\thediag}\frac{1}{\max(d(x,\thediag), d(y,\thediag))}\dd\gamma(x,y) \\& \leq \frac{1}{d(A, \thediag)}\iint_{A\times\overline{\groundspace}}\dd\gamma(x,y) = \frac{\mu(A)}{d(A,\thediag)} < + \infty,
\end{align}
and
\begin{align}
    \tilde\gamma(A \times\thediag) &= \iint_{A\times\thediag}\frac{\dd\gamma(x,y)}{d(x,\thediag)} + \iint_{A_+}\left(\frac{1}{d(x,\thediag)} - \frac{1}{d(y, \thediag)}\right)\dd\gamma(x,y) \\
    &\leq \frac{1}{d(A,\thediag)}\iint_{A\times\overline{\groundspace}}\dd \gamma(x,y) \\
    &= \frac{\mu(A)}{d(A,\thediag)} < +\infty,
\end{align}
which concludes the proof.
\end{proof}

We also have the following straightforward lemma.
\begin{lemma} \label{lem:int-proj-n}
    Let $x = (x_1,x_2) \text{ and } y = (y_1, y_2) \in \groundspace$.  
    We have
\begin{align} 
\int_\bR d(\pi_t(x), \pi_t(y))\dd t &\leq \sqrt2\max(d(x,\thediag), d(y, \thediag))\|x-y\|^p \\ &\qquad + \sqrt2\max(d(x,\thediag), d(y, \thediag))^p)\|x-y\|.
\end{align}

\end{lemma}

\begin{proof}
    Assume $x_1 \leq y_1 \leq x_2 \leq y_2$. Then for all $t \notin [x_1, y_2], d(\pi_t(x),\pi_t(y)) = 0$ and we have

\begin{equation} 
\int_{x_1}^{y_1}d(\pi_t(x), \pi_t(y))^p\dd t = d(x, \thediag)^p(y_1-x_1) \text{ and } \int_{x_2}^{y_2}d(\pi_t(x), \pi_t(y))\dd t = d(y, \thediag)^p(y_2-x_2).
\end{equation}
Furthermore, since for all $x_2 \leq t \leq y_1, d(\pi_t(x), \pi_t(y)) \leq \|x -y\|$, we also have
\begin{equation}
    \int_{y_1}^{x_2}d(\pi_t(x), \pi_t(y))^p\dd t \leq (x_2 - y_1)\|x-y\|^p \leq \max(d(x,\thediag), d(y,\thediag))\|x-y\|^p.
\end{equation}
Combining these inequalities yields the desired result. The calculation is the same when $x_1 \leq y_1 \leq y_2 \leq x_2$.
\end{proof}

We can now prove \cref{thm:ineg-projn}.

\begin{proof}
Let $\mu,\nu \in \cM^p(\groundspace) \cap \cM^\infty(\groundspace)$. Observe that $\supp(\gamma) \cap \{(x,y) \in \Omega^2 : x_1 \leq x_2 \leq y_1 \leq y_2 \text{ or } y_1 \leq y_2 \leq x_1 \leq x_2\} = \emptyset$. Hence, using \ref{lem:int-proj-n}, we get
\begin{align}
    \int_\bR\iint_{\Omega\times\Omega} d(\pi_t(x), \pi_t(y))^p\dd &\tilde\gamma(x,y) \dd t \leq\sqrt2\iint_{\groundspace\times\groundspace}\|x - y\|^p\dd \gamma(x,y) \\
    &\quad + \sqrt2\iint_{\groundspace\times\groundspace}\max(d(x,\thediag), d(y, \thediag))^{p-1}\|x - y\|\dd \gamma(x,y).
\end{align}
Then, one has
\clearpage
\begin{align}
    \int_\bR\iint_{\groundspace\times\thediag}d(\pi_t(x), \pi_t(y))^p\dd &\tilde\gamma(x,y)\dd t \leq \sqrt2\iint_{\groundspace\times\thediag}d(x,\thediag)^p\dd\gamma(x,y) \\
    &\quad + \sqrt2\iint_{\groundspace_{1,+}}d(x,\thediag)^{p+1}\left(\frac{1}{d(x,\thediag)} - \frac{1}{d(y,\thediag)}\right)\dd\gamma(x,y),
\end{align}
where $\groundspace_{1,+} = \{(x, y) \in \groundspace\times\groundspace: d(y,\thediag) \geq d(x, \thediag)\}$ and $\groundspace_{2,+} = \{(x, y) \in \groundspace\times\groundspace: d(x,\thediag) \geq d(y, \thediag)\}$. 
We also have a similar inequality when integrating over $\thediag\times\groundspace$. 
Combining those results, we get
\begin{align} \label{eq:inegProjnSharp}
    \begin{split}
    \sqrt2\SFG_p^p(\mu,\nu) &\leq \sqrt2\FG_p^p(\mu,\nu) \\
    &\quad + \sqrt2\iint_{\groundspace\times\groundspace}\max(d(x,\thediag), d(y, \thediag))^{p-1}d(x,y)\dd \gamma(x,y) \\
    &\quad + \sqrt2\iint_{\groundspace\times\groundspace} \frac{\min(d(x,\thediag),d(y,\thediag))^p}{\max(d(x,\thediag),d(y,\thediag))}d(x,y)\dd\gamma(x,y),
    \end{split}
\end{align}
from which we deduce the result of \cref{thm:ineg-projn}.
\end{proof}

It is unfortunately impossible to get a uniform upper bound (that is a bound given by a constant factor of the $\FG_p$ distance) when $p>1$. 
Indeed, consider $\mu_n = \delta_{(-n,n)}$ and $\nu_n = \delta_{(-n-\frac{1}{\ln(n)}, n+ \frac{1}{\ln(n)})}$ on one hand we get $\FG_p(\mu_n,\nu_n) = \frac{\sqrt{2}}{\ln(n)}$ and on the other we have
\begin{equation} 
\SFG_p^p(\mu_n,\nu_n) \geq \sqrt2^{p+1}n^{p+1}\left(\frac{1}{\sqrt2n}-\frac{1}{\sqrt2(n+\frac{1}{\ln(n)})}\right) \underset{n\xrightarrow{}+ \infty}{\sim} C\frac{n^{p-1}}{\ln(n)}.
\end{equation}
In particular, $\FG_p(\mu_n,\nu_n) \xrightarrow[n\xrightarrow{} +\infty ]{}0$ and $\SFG_p(\mu_n,\nu_n) \xrightarrow[n\xrightarrow{} +\infty ]{}+\infty$. \\ \\
Similarly, it is not possible to get a uniform lower bound of the form $C_p\SFG_p \leq \FG_p$. 
Indeed, by considering the following two diagrams $\mu_n = \{(k\frac{\sqrt2}{n}, (k+1)\frac{\sqrt2}{n}), k\in [0, \lceil n^p\rceil]\}$ and $\nu_n = \{(k+\frac{1}{2})\frac{\sqrt2}{n}, (k+1 + \frac{1}{2})\frac{\sqrt2}{n}), k\in [0, \lceil n^p\rceil\}$. 
We then have $\SFG_p^p(\mu_n, \nu_n)= \frac{1}{n^p}$ and $\FG_p(\mu_n, \nu_n) = \sum_{k=0}^{\lceil n^p\rceil} \frac{1}{n^p} \geq 1$. \\ \\
The above shows that $\SFG_p$ and $\FG_p$ are not strongly equivalent. We will therefore compare these two distances in a slightly weaker manner by instead proving that they are topologically equivalent.

\subsubsection{Equivalence of topologies}
In this section, we prove that the topology of $\cD^p(\groundspace)$ endowed with $\SFG_p$ is the same as with $\FG_p$. 
Note that in this specific subsection, our proof restrict to (possibly infinite) persistence diagrams instead of measures for the sake of simplicity. 
Extending the result to that more general setting is left for future work. 
\begin{thm}[The topologies defined by $\SFG_p$ and $\FG_p$ are the same] \label{thm:equivTopSFG}
    Let $\mu \in \cD^p(\groundspace)$ and $(\mu_n)_n$ be a sequence in $\cD^p(\groundspace)$. 
    Then
    \begin{equation}
        \mu_n \xrightarrow[]{\FG_p} \mu \iff \mu_n \xrightarrow[]{\SFG_p} \mu.
    \end{equation}
\end{thm}
We will use the following characterisation of convergence in $(\cM^p(\groundspace), \FG_p)$ from \citep{DL20},
\begin{equation} \label{eq:convergence}
    \FG_p(\mu_n, \mu) \xrightarrow{} 0 \iff
    \begin{cases}
    \mu_n \xrightarrow{v} \mu \\
    \mathrm{Pers}_p(\mu_n) \xrightarrow{} \mathrm{Pers}_p(\mu)
    \end{cases}
\end{equation}
Where $\mu_n \xrightarrow[]{v} \mu$ means that for all functions $f$ with compact support in $\groundspace$, we have $\int_\groundspace f(x)\dd \mu_n(x) \to \int_\groundspace f(x) \dd \mu(x)$.
\begin{prop}
    Let $\mu \in \cD^p(\groundspace)$ and $(\mu_n)_n$ be a sequence of $\cD^p(\groundspace)$ such that $\mu_n \xrightarrow[]{\FG_p} \mu$ then $\mu_n \xrightarrow[]{\SFG_p} \mu$.
\end{prop}
\begin{proof}
    If $p=1$ this is simply a consequence and \ref{thm:ineg-projn}. 
    For $p > 1$, we need to be slightly more precise.
    
    Consider a sequence $\mu_n \in \cD^p$ such that $\FG_p(\mu_n, \mu) \to 0$ and let $\varepsilon > 0$. Because of (\ref{eq:convergence}), we therefore have $\mu_n \xrightarrow[]{v} \mu$ and in particular $\mu_n$ converges pointwise to $\mu$. Hence, if we denote by $\gamma_n$ the optimal transport plan achieving the infimum in $\FG_p(\mu,\mu_n)$, for $n$ large enough we know that if $\gamma(x,y) > 0$ then $d(x,y) < \varepsilon^2$. We will now bound the third term in \ref{eq:inegProjnSharp}. Let $A_{\geq\varepsilon} = \{(x,y) \in \groundspace\times\groundspace : d(x,\thediag) \geq \varepsilon \text{ and } d(y, \thediag) \geq \varepsilon\}$ and $A_{\leq\varepsilon} \{(x,y) \in \groundspace\times\groundspace : d(x,\thediag) \leq \varepsilon \text{ and } d(y, \thediag) \leq \varepsilon\}$, we have
    \begin{align}
        \iint_{\groundspace\times\groundspace}\frac{\min(d(x,\thediag), d(y,\thediag))^{p}}{\max(d(x,\thediag),d(y,\thediag))}&d(x,y)\dd\gamma_n(x,y) \\ & = \iint_{A_{\geq\varepsilon}}\frac{\min(d(x,\thediag), d(y,\thediag))^{p}}
        {\max(d(x,\thediag),d(y,\thediag))}d(x,y)\dd\gamma_n(x,y) \\  
        & \quad + \iint_{A_{\leq2\varepsilon}}\frac{\min(d(x,\thediag), d(y,\thediag))^{p}}{\max(d(x,\thediag),d(y,\thediag))}d(x,y)\dd\gamma_n(x,y) \\
        &\leq \varepsilon^2\iint_{A_{\geq\varepsilon}}\frac{d(x,\thediag)^p}{\inf_{(z,w) \in A_{\geq\varepsilon}}(d(z,\thediag))}\dd\gamma_n(x,y) \\
        & \quad + \iint_{A_{\leq2\varepsilon}}\min(d(x,\thediag), d(y,\thediag))^{p}d(x,\thediag)\dd\gamma_n(x,y)\\
        & \quad + \iint_{A_{\leq2\varepsilon}}\min(d(x,\thediag), d(y,\thediag))^{p}d(y,\thediag)\dd\gamma_n(x,y)\\
        & \leq \varepsilon \mathrm{Pers}_p(\mu_{\geq \varepsilon}) + \mathrm{Pers}_p(\mu_{\leq2\varepsilon}) + \mathrm{Pers}_p({\mu_n}_{\leq2\varepsilon}).
    \end{align}
    Now, using once again \ref{eq:convergence}, we know that $\mathrm{Pers}_p({\mu_n}_{\leq 2\varepsilon}) \to \mathrm{Pers}_p(\mu_{\leq 2\varepsilon})$ and so we finally get that, for $n$ sufficiently large,
    \begin{equation}
         \iint_{\groundspace\times\groundspace}\frac{\min(d(x,\thediag), d(y,\thediag))^{p}}{\max(d(x,\thediag),d(y,\thediag))}d(x,y)\dd\gamma_n(x,y) \leq \varepsilon(1 + \mathrm{Pers}_p(\mu)) + 2\mathrm{Pers}_p(\mu_{\leq \varepsilon}).
    \end{equation}
    A similar bound can be obtained for the second term in \ref{eq:inegProjnSharp}. Combining these together yields the desired result i.e. that $\mu_n \xrightarrow[]{\SFG_p} \mu$.
\end{proof}

We now prove the converse implication. 
For this, we will show that the $\SFG_p$ convergence implies characterisation \ref{eq:convergence}.
\begin{prop}
    Let $\mu \in \cD^p(\groundspace)$ and $(\mu_n)_n$ be a sequence of $\cD^p(\groundspace)$ such that $\mu_n \xrightarrow[]{\SFG_p} \mu$. Then $\mathrm{Pers}_p(\mu_n) \to \mathrm{Pers}_p(\mu)$.
\end{prop}
\begin{proof}
    From Proposition \ref{prop:sfg-pers}, we have $\SFG_p(\nu, \emptyset) = \mathrm{Pers}_p(\nu)$ for all $\nu \in \cM^p(\groundspace)$. The result then follows from the triangle inequality.
\end{proof}

It remains to prove that the $\SFG_p$ convergence implies the vague convergence in duality with functions of $\cC_c(\groundspace)$. To do this, we will prove the pointwise convergence of the points of $\mu_n$ to the points of $\mu$ from which we will deduce the result using the fact that a compact of $\groundspace$ only contains finitely many of those points.

\begin{prop} \label{prop:top-SFG}
    Let $\mu \in \cD^p(\groundspace)$ and $(\mu_n)$ be a sequence of $\cD^p(\groundspace)$ such that $\mu_n \xrightarrow[]{\SFG_p}\mu$. Then for all $x \in \supp(\mu)$, there exists $\eta > 0$ such that for all $0 < \varepsilon < \eta$, $\mu_n(1_{B(x, \varepsilon)}) \to \mu(x)$
\end{prop}
To prove this result, we need the following technical lemma and definition.
\begin{definition} \label{def:technical-topn}
     Given $\varepsilon > 0$ and $x \in \groundspace$, we denote by $B_\varepsilon^x$ (or simply $B_\varepsilon$) the strip parallel to the diagonal of width $2\varepsilon$ and centered at $x$, by $\partial B_\varepsilon$ its border and by $B_{\varepsilon,t} \coloneqq B_\varepsilon \cap \groundspace_t$. 
     
     Given $\mu \in \cM^p(\groundspace)$ and $C \subset \groundspace$, we also define $\mu_C$ by $\mu_C(A) = \mu(A \cap C)$.
     Now, let $\mu,\nu \in \cM^p(\groundspace)$ and $C, D \subset \groundspace$ we define
     \begin{equation} 
     \tilde\FG_{C,D}(\mu, \nu) \coloneqq \inf_{\gamma \in \tilde\Adm(\mu_C,\nu_D)} \iint_{\overline{C} \times \overline{D}}d(x,y)^p \dd \gamma(x,y),
     \end{equation}
     where $\tilde\Adm(\mu_C,\nu_D)$ is the set of Radon measures supported on $\overline C \times \overline D$ satisfying that, for all Borel subsets $A \subset C$ and $B \subset D$,
     \begin{equation} 
     \gamma(A \times \overline D) = \mu_C(A) \quad \text{ and } \quad \gamma(\overline C \times B) = \nu_D(B).
     \end{equation}
     Eventually, we define
     \begin{equation} 
     \tilde\SFG_{C, D}(\mu,\nu) \coloneqq \frac{1}{\sqrt2}\left(\int_\bR\tilde\FG_{C,D}^p(\pi_t\#\tilde\mu,\pi_t\#\tilde\nu)\dd t\right)^{\frac{1}{p}}.
     \end{equation}
\end{definition}

\begin{lemma} \label{lem:trop-projn}
    Let $\mu, \nu \in \cM^p(\groundspace)$, for any $\varepsilon, \eta > 0$ and $x \in \groundspace$ we have $\tilde\SFG_{B_\varepsilon, B_\eta}(\mu,\nu) \leq \SFG_p(\mu,\nu)$.
\end{lemma}
\begin{proof}
    Let $\varepsilon, \eta > 0$ and $t \in \bR$. We denote by $\gamma_t$ one of the pullbacks of the optimal transport plan achieving the infimum in $\FG_p(\pi_t\#\tilde\mu, \pi_t \#\tilde\nu)$. We can then write
    \begin{align}
        \iint_{\bar\Omega \times \bar\Omega} d(\pi_t&(x), \pi_t(y))^p \dd \gamma_t(x,y) \\ &\geq  \iint_{B_\varepsilon \times B_{\eta} } d(\pi_t(x), \pi_t(y)) \dd \gamma_t(x,y) +  \iint_{B_\varepsilon \times B_{\eta}^c \cup B_\varepsilon^c \times B_{\eta}} d(\pi_t(x), \pi_t(y)) \dd \gamma_t(x,y) \\ \\
        &\geq \iint_{B_\varepsilon \times B_{\eta} } d(\pi_t(x), \pi_t(y)) \dd \gamma_t(x,y) + \iint_{B_\varepsilon \times B_{\eta}^c} d(\pi_t(x), \partial B_{\eta})\dd \gamma_t(x,y) \\ 
        & \quad + \iint_{B_\varepsilon^c \times B_{\eta}} d(\partial B_\varepsilon, \pi_t(y)) \dd \gamma_t(x,y) \\
        &= \iint_{\bar B_\varepsilon \times \bar B_{\eta}} d(\pi_t(x), \pi_t(y)) \tilde \gamma_t(x,y),
    \end{align}
    where we set
    \begin{equation}
    \tilde\gamma_t(x,y) = 
        \begin{cases}
            \gamma_t(x,y) & \text{if } x,y \in B_\varepsilon \times B_{\eta} \\
            \int_{z \in B_{\eta}^c} \dd \gamma_t(x,z) & \text{if } x \in B_\varepsilon \text{ and } y = \partial B_{\eta} \\
            \int_{z\in B_\varepsilon^c} \dd \gamma_t(z,y) & \text{if } x= \partial B_\varepsilon \text{ and } y \in B_{\eta}.
        \end{cases}
    \end{equation}
    It is then easy to check that $\pi_t\#\tilde\gamma_t \in \tilde\Adm(\pi_t\#\tilde\mu_{B_\varepsilon}, \pi_t\#\tilde\nu_{B_\eta})$ for all $t \in \bR$ which concludes.
\end{proof}

\begin{lemma} \label{lem:technical-top-n}
     Let $I \subset \bR$, $\mu, \nu \in \cD^p(\groundspace)$ and $x \in  \groundspace$. For any $t \in \bR$ and $\varepsilon > 0$, let $\Delta_t \coloneqq \tilde\nu(B_{\varepsilon/2,t}) - \tilde\mu(B_{\varepsilon/2,t})$.
     
     There exists $\eta > 0$ depending only on $\mu$ and $x$, such that for all $0 < \varepsilon < \eta$ and $C > 0$, if $\Delta_t > C$ for all $t \in I$ then
     \begin{equation} 
     \tilde\SFG_{B_{\varepsilon/2}, B_\varepsilon}(\nu, \mu) \geq C|I|\varepsilon/2.
     \end{equation}
\end{lemma}

\begin{proof}
    Since there are only finitely many points of $\mu$ in $B_{d(x,\thediag)/2}$, there exists $\eta>0$ verifying $\supp(\mu) \cap B_\eta \subset D_x \coloneqq \{z \in \groundspace : d(z, \thediag) = d(x, \thediag)\}$. 
    Let $0 <\varepsilon < \eta$. We have
\begin{align}
    \tilde\SFG_{B_{\varepsilon/2}, B_\varepsilon}(\nu,\mu) &\geq \int_I\tilde\FG_{B_{\varepsilon/2}, B_\varepsilon}(\pi_t\#\tilde\nu, \pi_t\#\tilde\mu)\dd t \\
    & \geq \int_I \Delta_t\times \varepsilon/2 \dd t \\
    &\geq C|I|\varepsilon/2.
\end{align}
The second inequality comes from the fact that since $|\Delta_t| > C$ there are two cases:
\begin{enumerate}
    \item if $\Delta_t > C$, then at least $C$ amount of mass of $\nu$ needs to be transported from $B_{\varepsilon/2}$ to $\partial B_\varepsilon$ hence a cost of at least $\varepsilon/2$,
    \item if $\Delta_t < -C$, then at least $C$ amount of mass of $\mu$ needs to be transported from $D_x$ to $\partial B_{\varepsilon/2}$ hence a cost of $\varepsilon/2$. 
\end{enumerate}
\end{proof}

We can now prove Proposition \ref{prop:top-SFG}. For the sake of simplicity we will assume $p=1$ but the argument can easily be adapted to the general case.

\begin{proof}
    \textbf{First step: $\liminf \mu_n(1_{B(x, \varepsilon)}) \geq \mu(x)$}

    Suppose for a contradiction that there exists $x \in \supp(\mu)$ such that for all $\eta > 0$ there exists $0 < \varepsilon < \eta$ satisfying that for infinitely many $\mu_n$ we have $\mu_n(1_{B(x, 2\varepsilon)}) < \mu(x)$. Let $\eta$ be given by Lemma \ref{lem:technical-top-n} and $0 < \varepsilon < \eta$. By further reducing $\eta$, we can assume that $\supp(\mu) \cap B_\eta \subset D_x \coloneqq \{z \in \groundspace : d(z, \thediag) = d(x, \thediag)\}$ and that the minimal distance between two points of $D_x \cap \supp(\mu)$ is at least $2\varepsilon$. Since according to Lemma \ref{lem:trop-projn} we have $\tilde\SFG_{\varepsilon/2, \varepsilon}(\mu_n,\mu) \leq \SFG(\mu_n,\mu)$ we will lower bound $\tilde\SFG_{B_{\varepsilon/2}, B_\varepsilon}(\mu_n,\mu)$ (which we will denote simply by $\tilde\SFG$ for short) to obtain a contradiction.

\begin{figure}
    \centering
    \includegraphics[width=0.5\linewidth]{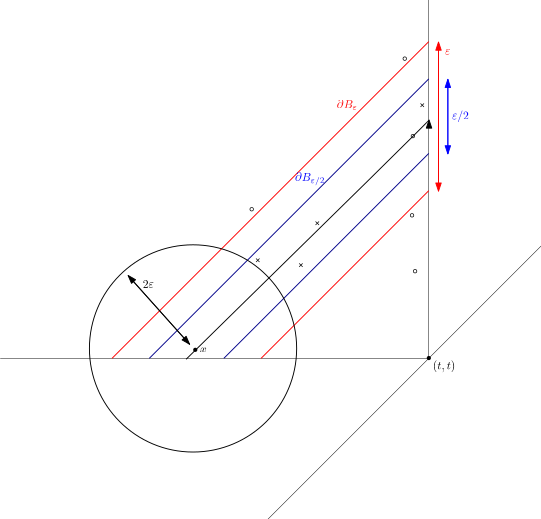}
    \caption{$B_{\varepsilon/2}$ (blue) and $B_\varepsilon$ (red)}
    \label{fig:projn-B-epsilon}
\end{figure}

In the rest of the proof, we use the result of Lemma \ref{lem:technical-top-n} repeatedly to study more extensively the quantity $\Delta_t$ (introduced in the same lemma) and reach a contradiction. Also, all points of $\supp(\tilde\mu_n) \cap B_{\varepsilon/2}$ do not have the same mass (they range between $\frac{1}{d(x,\thediag) + \varepsilon}$ and $\frac{1}{d(x,\thediag) - \varepsilon}$). However, since $\mathrm{Pers}_p(\mu_n) \to \mathrm{Pers}_p(\mu)$, the sequence $\tilde\mu_n(B_{\varepsilon/2})$ is uniformly bounded by some $M > 0$ and hence one can show that $\Delta_t$ lies within $\frac{2M\varepsilon}{d(x,\thediag)^2}$ of the nearest integer multiple of $1/d(x,\thediag)$. As such, by possibly shrinking $\eta$ further, this quantity can be assumed to be smaller than $\delta := \frac{1}{4d(x,\thediag)}$. \\
\\
Observe that for $t\in [x_2 - \varepsilon, x_2 + \varepsilon]$, at most $\mu(x) - 1$ points (counting multiplicities) of $\mu_n$ "disappear" from $B_{\varepsilon/2,t}$ (that is are contained in $B_{\varepsilon/2, x_2 - \varepsilon}$ but not in $B_{\varepsilon/2, x_2 + \varepsilon}$) and during this interval, the only point of $\mu$ leaving $B_{\varepsilon/2,t}$ is $x$ (of multiplicity $\mu(x)$). Then, during that same interval, only one point of $\mu$ can appear in $B_{\varepsilon/2,t}$ which we will denote by $y$. We distinguish several cases:
\begin{enumerate}
    \item At $t_0 = x_2 - \varepsilon$, $|\Delta_{t_0}| \leq \delta$. We then distinguish 3 cases depending on when $y$ appears.
    \begin{enumerate}
        \item If $y$ never appears, then since $x$ disappears at $t = x_2$ and at most $\mu(x) - 1$ points of $\mu_n$ (counting multiplicities) disappear between $t = x_2 - \varepsilon$ and $t = x_2 +\varepsilon$ we have $\Delta_t \geq 1/d(x,\thediag) - \delta$ for $t\in [x_2, x_2 +\varepsilon]$ and therefore, according to \cref{lem:technical-top-n}, $\tilde\SW(\mu_n,\mu) \geq \frac{\varepsilon^2}{2}(\frac{1}{d(x,\thediag)} - \delta)$.
        \item Then, if $y$ appears before $x_2 - \varepsilon/4$, $\mu(y)$ points of $\mu_n$ must appear before $x_2$ otherwise $\Delta_t \leq -1/d(x,\thediag) + \delta$ for $t \in [x_2 - \varepsilon/4, x_2]$. So we again have $\Delta_t \geq 1/d(x,\thediag) - \delta$ for $t \in [x_2, x_2 + \varepsilon]$ hence $\tilde\SW(\mu_n,\mu) \geq \frac{\varepsilon^2}{2}(\frac{1}{d(x,\thediag)} - \delta)$. The case where $y$ appears after $x_2 + \varepsilon/4$ is handled similarly.
        \item If $y$ appears between $x_2 - \varepsilon/4$ and $x_2 + \varepsilon/4$. Then, if $\mu(y)$ points of $\mu_n$ or more appear between $x_2 - 3\varepsilon/4$ and $x_2 + 3\varepsilon/4$, we get this time $\tilde\SW\geq \frac{\varepsilon^2}{8}(\frac{1}{d(x,\thediag)} - \delta)$ since in that case we would have $\Delta_t \geq 1/d(x,\thediag) - \delta$ between $x_2 + 3\varepsilon/4$ and $x_2 + \varepsilon$ (this is because the distance between two points of $D_x \cap \supp(\mu)$ is at least $\varepsilon$). As such less than $\mu(y)$ points of $\mu_n$ appear between $x_2 - 3\varepsilon/4$ et $x_2 + 3\varepsilon/4$. In other words, $y$ verifies the same hypothesis as $x$ but with $\varepsilon/2$ instead of $2\varepsilon$ i.e. $\mu_n(1_{B(y, \varepsilon/2)}) < \mu(y)$.
    \end{enumerate}
    \item If at $t_0=x_2 - \varepsilon$, $|\Delta_{t_0}| \geq \delta$, then there exists $\varepsilon' \in [3\varepsilon/4, \varepsilon]$ such that $|\Delta_{x_2 - \varepsilon'}| \leq \delta$. We are thus reduced to the previous case and we get that the point $y$ (if it exists) verifies $\mu_n(1_{B(y, \varepsilon'/4)}) < \mu(y)$ and so in particular $\mu_n(1_{B(y, \varepsilon/4)}) < \mu(y)$.
\end{enumerate}
Setting $N \coloneqq |D_x \cap \mathrm{supp(\mu)|}$, it follows, by iteration, that $\tilde\SW(\mu_n,\mu) \geq \left(\frac{1}{8}\right)^{2N}\frac{\varepsilon^2}{2}(\frac{1}{d(x,\thediag)} - \delta)$.
And so we deduce that $\SW_p(\mu_n,\mu) \not\to 0$ which is the desired contradiction. Hence, for all sufficiently large $n$, we have $\mu_n(1_{B(x, \varepsilon)}) \geq \mu(x)$. \\

\textbf{Second step: $\limsup \mu_n(1_{B(x, \varepsilon)}) \leq \mu(x)$}: \\
Suppose for a contradiction that, up to extracting a subsequence of $(\mu_n)$, there exist sequences $(a_{i,n})_{1\leq i \leq k}$ of points of $\supp(\mu_n)$ such that $a_{i,n} \xrightarrow[n\to +\infty]{} x$, the $a_{i,n}$ are all distincts and $\sum_{i= 1}^k \mu_n(a_{i,n}) > \mu(x)$ for all $n \in \bN$. Like above, we consider the problem $\tilde\SFG$ associated to the regions $B_\varepsilon$ and $B_{\varepsilon/2}$ where we assume $\varepsilon$ is small enough so that the only points of $B_\varepsilon\cap\supp(\mu)$ are in $D_x$ and the distance between two points of $D_x\cap\supp(\mu)$ is at least $2\varepsilon$. We then fix $N$ large enough to have $d(a_{i,n},x) \leq \varepsilon/4$ for all $n \geq N$ and all $1 \leq i \leq k$. Furthermore, using the result of the first step, for all $z \in D_x$, we may find sequences $(z_{i,n})_{1\leq i \leq k_z}$ such that $z_{i,n} \to z$, the $z_{i,n}$ are all distincts and $\sum_{i= 1}^{k_z} m_{\mu_n}(z_{i,n}) \geq \mu(z)$ for all $n \in \bN$. Hence, we can assume that $d(z_{i,n}, z) \leq \varepsilon/4$ for all $z \in D_x, 1\leq i \leq k_z$ and $n \geq N$  (which is possible because $D_x\cap \supp(\mu)$ is finite). 

In that setting, with $t_0 = x_2 - \varepsilon$, to any point $z \in (D_x \cap B_{\varepsilon/2, t_0}) \setminus \{x\}$ except eventually the "last" (that is the one with the highest $z_2$ coordinate) corresponds at least $\mu(z)$ points (counting multiplicities) in $B_{\varepsilon/2,t_0}$ (the $z_{i,n}$ we defined above). To those are added the $\mu(x) + 1$ points associated to $x$ since we are arguing by contradiction. In particular, from this we deduce that if $\Delta_{t_0} \leq \delta$, the last point $y$ of $\supp(\mu) \cap B_{\varepsilon,t}$ appeared between $t_0 - \varepsilon/4$ and $t_0$. 
We then have
\begin{itemize}
    \item If $\Delta_{t_0} \leq \delta$, then (as explained above) the last point $y$ of $\supp(\mu)\cap B_{\varepsilon,t}$ appeared between $t_0 - \varepsilon/4$ and $t_0$ and as such the corresponding points $y_{i,n}$ must appear before $t_0 + \varepsilon/4$. We then have $\Delta_t \geq \frac{1}{d(x,\thediag)} - \delta$ for $x_2 - 3\varepsilon/4 \leq t \leq x_2 -\varepsilon/4$ which is not possible.
    \item If $\Delta_{t_0} \geq \frac{1}{d(x,\thediag)} - \delta$, then necessarily a point $y$ of $\mu$ appears before $t = x_2 - 3\varepsilon/4$. But then again, the associated points $y_{i,n}$ appear before t = $x_2 - \varepsilon/2$. However no other point of $\supp(\mu)$ can appear before $x_2 - \varepsilon/4$ and as such $\Delta_t \geq 1/d(x,\thediag) - \delta$ for $x_2 - \varepsilon/2 \leq t \leq x_2 -\varepsilon/4$ hence $\tilde\SW \geq \frac{\varepsilon^2}{8}(\frac{1}{d(x,\thediag)} - \delta)$.
\end{itemize}
Finally, for all $x \in \supp(\mu)$, there exists $\eta > 0$ such that for all $\eta > \varepsilon > 0$ we have $\mu_n(1_{B(x, \varepsilon)}) \to \mu(x)$.
\end{proof}

We finally get the following result. 
\begin{prop}
     Let $\mu \in \cD^p(\groundspace)$ and $(\mu_n)_n$ be a sequence of $\cD^p(\groundspace)$ such that $\mu_n \xrightarrow[]{\SFG_p} \mu$. Then $\mu_n \xrightarrow[]{v} \mu$.
\end{prop}
\begin{proof}
    This is a direct consequence of \cref{prop:top-SFG} and the fact that any compact set of $\groundspace$ contains finitely many points of $\supp(\mu)$.
\end{proof}
And hence the converse implication in \cref{thm:equivTopSFG} holds.

\section{Numerical considerations and experimental results}\label{sec:expe}

In this section, we introduce algorithms for computing the Sliced Figalli–Gigli distance, and the associated kernel (SFGK), both in exact and approximate forms. We then evaluate their empirical performance on standard machine learning tasks, using the same benchmark datasets as those considered for the Sliced Wasserstein kernel on persistence diagrams proposed by Carrière et al.~\citep{pmlr-v70-carriere17a}.
Owing to the close similarity between the two kernels, we do not anticipate substantial differences in practical performance, a hypothesis that is confirmed by our numerical experiments.
Accordingly, the goal of this section is not to claim a systematic empirical advantage over the SWK of~\citep{pmlr-v70-carriere17a}, but rather to demonstrate that the stronger theoretical foundations underpinning SFGK do not come at the expense of numerical efficiency or predictive performance.

\subsection{Exact and approximated computation of the SFG distance}
In this section we give algorithms to compute both the exact value and approximations of $\SFG$. 

\begin{algorithm} \label{alg:exact-SFG}
    \caption{Exact computation of $\SFG_p$}
    \begin{algorithmic}
        \State \textbf{Input:} Two diagrams $\rX= \{x_1, \dots x_{N_1}\}, \rY= \{y_1, \dots y_{N_2}\}$
        \State Sort $\{x_{1,1}, x_{1,2}, \dots , x_{N_1, 1}, x_{N_1,2}, y_{1,1}, \dots y_{N_2,2}\}$ into a list of events $\rE$.
        \State Initialize $\SFG = 0$ and two binary search trees $\rX_t = \{\}, \rY_t = \{\}$
        \For{$i = 1, \dots, 2(N_1 + N_2)$}
        \If{$\rE[i]$ is of the form $z_{j,1}$} 
        \State $\rZ_t \leftarrow \rZ_t + \{z_j\}$
        \Else 
        \State $\rZ_t \leftarrow \rZ_t - \{z_j\}$  \EndIf
        \State Iterate over $\rX_t$ and $\rY_t$ to compute $\alpha =\FG_p(\pi_t\#\tilde\rX, \pi_t\#\tilde\rY)$
        \State $\SFG \leftarrow \SFG + (\rE[i+1] - \rE[i])\alpha$
        \EndFor
        \State \textbf{Output:} $\SFG$
    \end{algorithmic}
\end{algorithm}

This algorithm works since the value of $\FG_p(\pi_t\#\tilde\rX, \pi_t\#\tilde\rY)$ does not change between two consecutive events. It computes the exact $\SFG_p$ distance between the two inputted diagrams in $\cO((N_1 + N_2)\times\max(N_1, N_2))$ where $N_1$ and $N_2$ are the number of points in each diagram. 
This is essentially optimal if $N_1$ and $N_2$ are roughly the same size since the above algorithm needs to compute the matrix of the distances between points of $X$ and points of $Y$ (any point of $X$ may be transported onto any point of $Y$) which can only be done in $\cO(N_1 \times N_2)$. 
We can also compute an approximation of the $\SFG_p$ distance by simply approximating the outer integral, as detailed in  \cref{alg:approx-SFG}.

\begin{algorithm} \label{alg:approx-SFG}
    \caption{Approximate computation of $\SFG_p$}
    \begin{algorithmic}
        \State \textbf{Input:} Two diagrams $\rX= \{x_1, \dots x_{N_1}\}, \rY= \{y_1, \dots y_{N_2}\}$ and an integer $k$
        \State Sort $\rX$ and $\rY$ with respect to $d(\cdot, \thediag)$
        \State Samples values $t_1, \dots, t_k$ between $\min(\{x_{i,1}, 1\leq i \leq N_1\} \cup \{y_{j,1}, 1\leq j \leq N_2\})$ and $\max(\{x_{i,2}, 1\leq i \leq N_1\} \cup \{y_{j,2}, 1\leq j \leq N_2\})$
        \State Initialize $\SFG = 0$
        \For{$i = 1, \dots, k$}
        \State Compute $\rX_t \coloneqq \rX \cap \groundspace_t$ and $\rY_t \coloneqq \rY \cap \groundspace_t$
        \State Iterate over $\rX_t$ and $\rY_t$ to compute $\alpha =\FG_p(\pi_{t_i}\#\tilde\rX, \pi_{t_i}\#\tilde\rY)$
        \State $\SFG \leftarrow \SFG + (t_{i} - t_{i-1})\alpha$
        \EndFor
        \State \textbf{Output:} $\SFG$
    \end{algorithmic}
\end{algorithm}

It is sufficient to sample the values $t_1, \dots, t_k$ between $\min(\{x_{i,1}, 1\leq i \leq N_1\} \cup \{y_{j,1}, 1\leq j \leq N_2\})$ and $\max(\{x_{i,2}, 1\leq i \leq N_1\} \cup \{y_{j,2}, 1\leq j \leq N_2\})$ since $\FG_p(\pi_t \# \tilde\rX, \pi_t \# \tilde\rY) = 0$ outside of these bounds. 
We considered sampling the $t_i$, either uniformly or by sampling from a Gaussian KDE fitted on the different event values (as defined in the previous algorithm). In practice, we found that the uniform sampling lead to better convergence speed (see Figure 4). The above algorithm runs in $\cO(N_1(\log(N_1) +k) + N_2(log(N_2) +k))$ since computing $\FG_p(\pi_t\#\tilde\rX, \pi_t\#\tilde\rY)$ can be done in $\cO((N_1 + N_2)k)$.

\begin{figure}[!h] \label{fig:convergence-sampling}
    \centering
    \begin{subfigure}{0.49\textwidth}
        \centering
        \includegraphics[width = 0.8\linewidth]{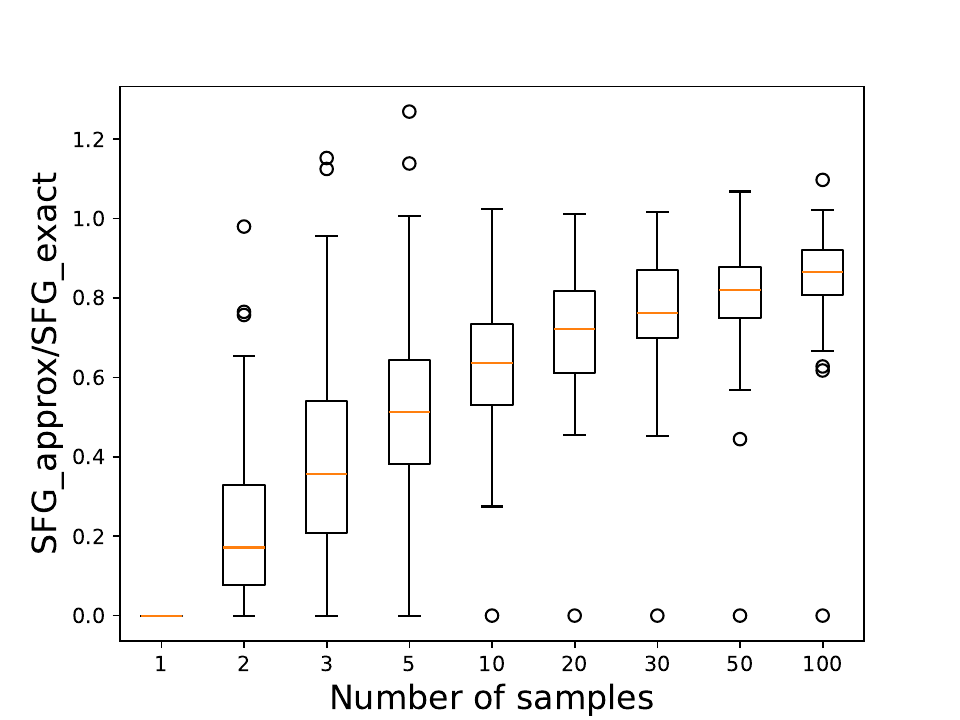}
        \caption{KDE sampling}
        \label{fig:sampling-kernel}
    \end{subfigure}
    \hfill
    \begin{subfigure}{0.49\textwidth}
        \centering
        \includegraphics[width=0.8\linewidth]{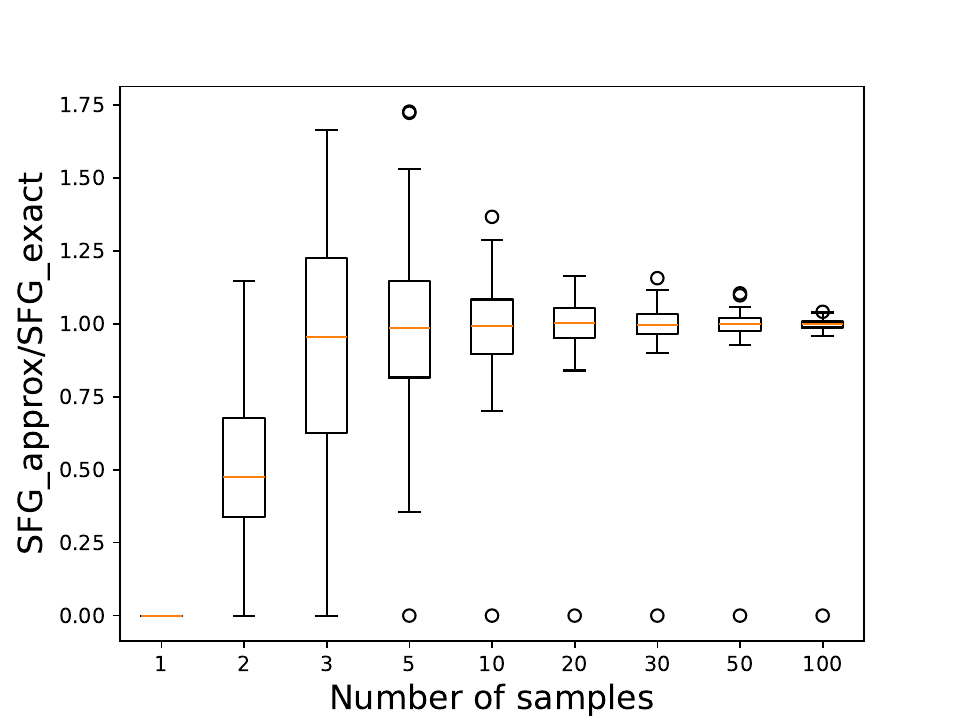}
        \caption{Uniform sampling}
        \label{fig:sampling-uniform}
    \end{subfigure}
    \caption{Boxplot of the ratio between the approximation of $\SFG$ and its exact value with respect to the number of samples used for the approximation. The persistence diagrams we used to get this plot were taken from the Orbits dataset we generated (see \cref{subs:experimental-results}). 
    The uniform sampling leads to better convergence speeds while the gaussian KDE sampling tends to underestimate the value of $\SFG$.}
\end{figure}
\begin{rem}
    In the above algorithm, the sampled values depend on the input diagrams. This allows to avoid sampling values where $\FG_p(\pi_t\#\tilde\rX, \pi_t\#\tilde\rX)$ is trivially zero. However, by doing that, this approximation of $\SFG_p$ is not CND. If one wants to ensure this property still holds when computing approximation of $\SFG_p$ the sampled values need to be fixed ahead of time. For example, when computing a Gram matrix of diagrams with the $\SFG$ kernel the sampled values used must be the same for all the distance computations or the resulting matrix might not be PSD. In practice, we found that when training an SVM the resulting matrix were close enough to being PSD for it to not cause convergence issues.
\end{rem}
\begin{rem}
    We also have very similar algorithms to compute exact and approximate values of the variant $\widehat\SFG$ described in \cref{appendix:continuous_proj} which we do not detail here. 
    They run in respectively $\cO(N_1N_2\times \max(N_1, N_2))$ and $\cO(N_1(\log(N_1) +k) + N_2(log(N_2) +k))$.
\end{rem}

\subsection{Tightness of theoretical bounds between $\FG$ and $\SFG$} \label{appendix:complementary-expe}
We explore numerically how tight the bound of \cref{thm:ineg-projn} is. 
For this purpose, we compute the ratio between the $\SFG$ distance between persistence diagrams sampled either (a) uniformly\footnote{That is, persistence diagrams each consisting of 100 points sampled uniformally on $[0,1]^2$.}, (b) from the \textit{Orbits} dataset or (c) from the \textit{Outex} dataset (see \cref{subs:experimental-results}) and the theoretical bound given by \cref{thm:ineg-projn}. 
The results are given in \cref{fig:bounds}.

We observe that the theoretical bound seems relatively tight for $p \neq 1$ as evidenced in \ref{fig:bounds_orbits}. 
For $p=1$ we empirically observe that $0.2 \lesssim \frac{\SFG_1}{3\FG_1} \lesssim 0.4$, i.e. $0.6 \FG_1 \lesssim \SFG_1 \lesssim 1.2\FG_1$. 
Even though we know it is not possible to obtain a lower bound of the form $C \cdot \SFG_1 \leq \FG_1$ for all persistence diagrams by considering peculiar adversarial examples, we observe that such a bound seems to hold for diagrams actually encountered in applications.

\begin{figure}[!h]
    \centering
    \begin{subfigure}{0.32\textwidth}
        \centering
        \includegraphics[width = \linewidth]{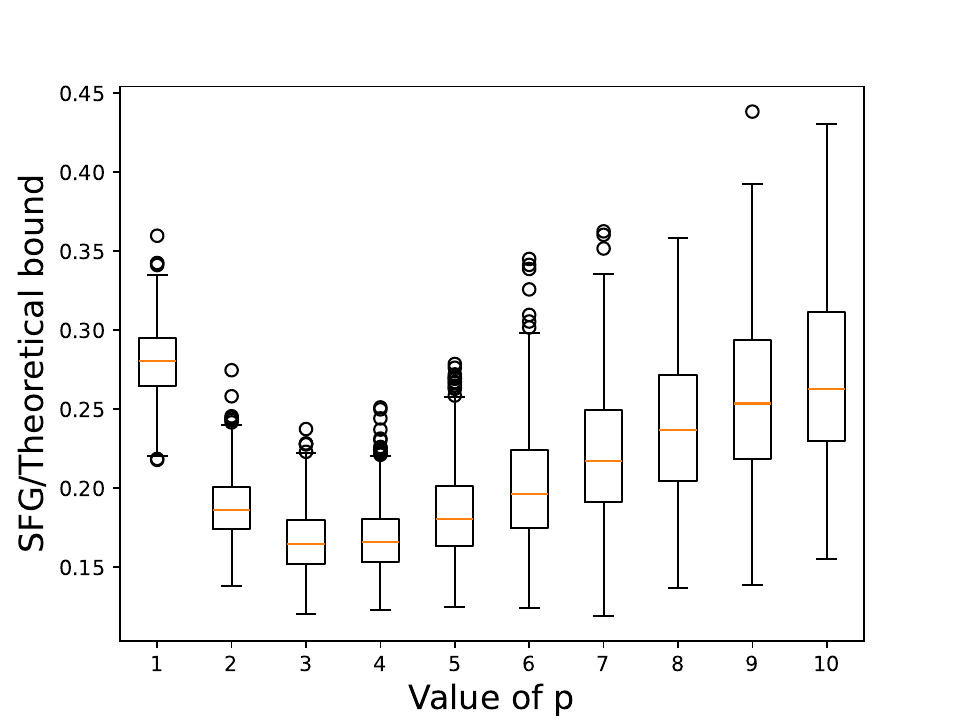}
        \caption{Uniform sampling}
        \label{fig:bounds_uniform}
    \end{subfigure}
    \hfill
    \begin{subfigure}{0.32\textwidth}
        \centering
        \includegraphics[width=\linewidth]{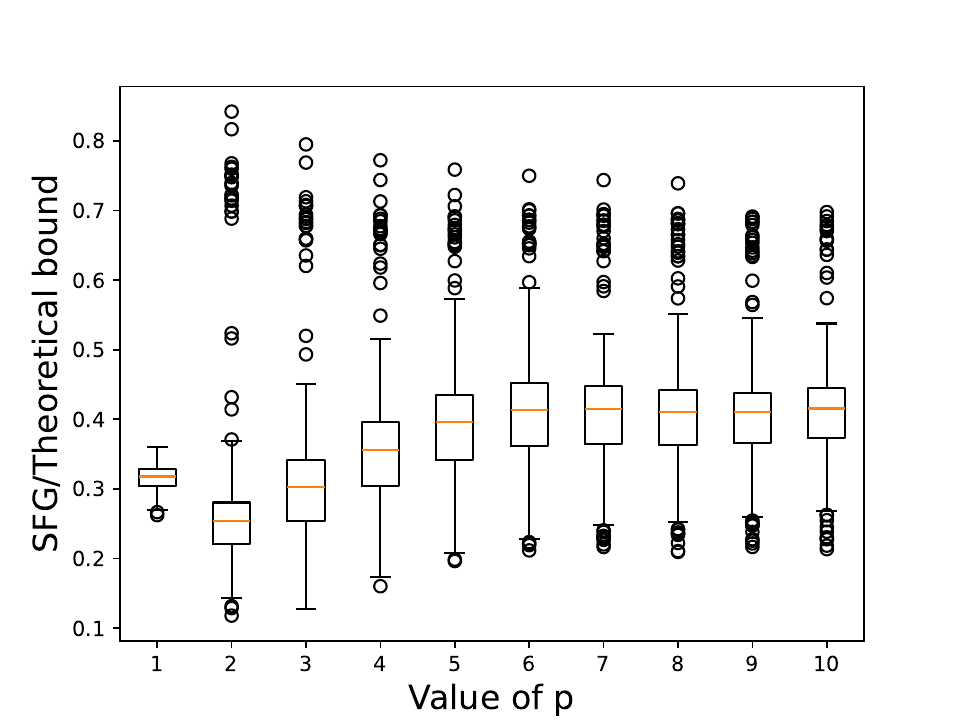}
        \caption{Orbits dataset}
        \label{fig:bounds_orbits}
    \end{subfigure}
     \begin{subfigure}{0.32\textwidth}
        \centering
        \includegraphics[width=\linewidth]{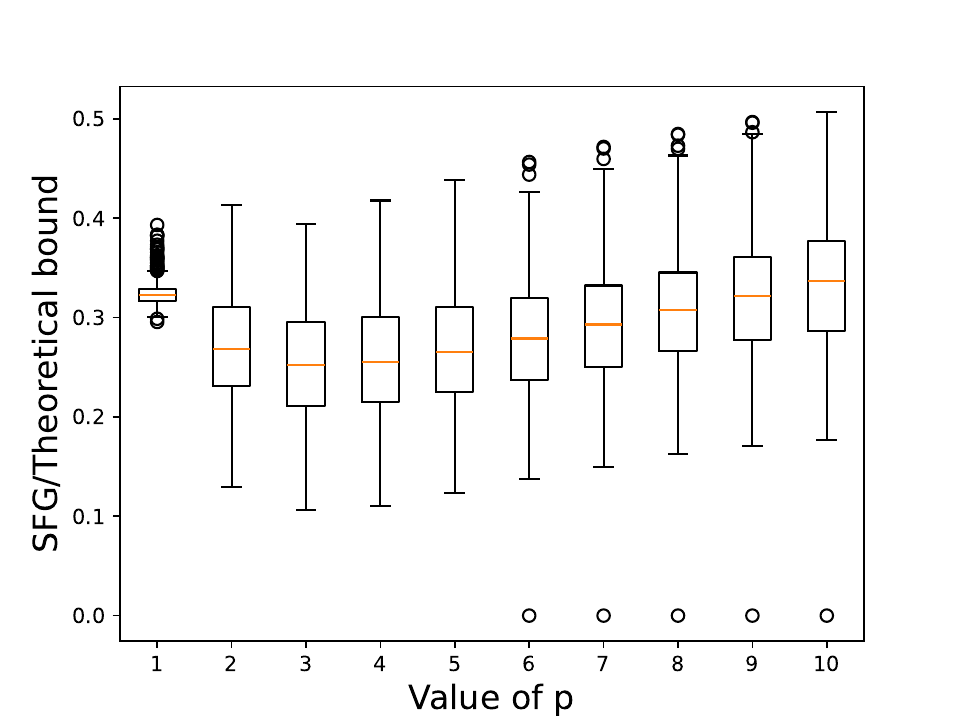}
        \caption{Outex dataset}
        \label{fig:bounds_texture}
    \end{subfigure}
    \caption{Boxplot of the ratio between $\SFG$ and the theoretical bound of \cref{thm:ineg-projn} for different values of $p$. For each value of $p$, the boxplot was obtained with 500 persistence diagrams which were sampled uniformally for (a) and using the procedures described in \cref{subs:experimental-results} for (b) and (c).}
    \label{fig:bounds}
\end{figure}

\subsection{Experimental results on classification tasks} \label{subs:experimental-results}
We tested the two distances $\SFG_1$ and $\widehat\SFG_1$ (simply denoted $\SFG$ and $\widehat\SFG$ below, where we recall that $\widehat\SFG$ refers to the variant of $\SFG$ using a different projection onto the geodesics, presented in \cref{appendix:continuous_proj}) against the Sliced Wasserstein distance ($\SW$, see section \ref{subs:related-works}) on two different experiments proposed in \citep{pmlr-v70-carriere17a}. For both of those, we train classifiers using the LIBSVM implementation \citep{libsvm} of $C$-SVM and average the results over 5 runs. We cross-validate the cost factor over the grid $\{0.01, 0.1, 1.0, 10.0, 100.0\}$. 
All the distance functions we consider only have one parameter: the bandwidth $\sigma$. When training, we choose it by 5-fold cross-validation among 15 different values that we obtain by multiplying the first decile, last decile and median of the Gram matrix values of the training set by the following factors: $\{0.01, 0.1, 1.0, 10.0, 100.0\}$. We use 20 directions to approximate $\SW$ and sample 20 values to approximate $\SFG$ and $\widehat\SFG$.

\subsubsection{Orbit recognition}
\textbf{Dataset} This task is based on the \textit{linked twist map} discrete dynamical system. Given initial positions $(x_0, y_0) \in [0,1]^2$ and a parameter $r > 0$ its orbits can be computed with the following equations:
\begin{equation}
    \begin{cases}
        x_{n+1} = x_n + ry_n(1-y_n) \mod 1 \\
        y_{n+1} = y_n + rx_{n+1}(1-x_{n+1}) \mod 1
    \end{cases}
\end{equation}
The orbits exhibit very different behaviour depending on the values of the parameter $r$. For example, as can be seen in Figure 5, when $r = 3.5$ there seems to be no specific structure whereas for $r= 4.1$ a void appears in the center.

\begin{figure}[!h] \label{fig:orbits}
    \centering
    \begin{subfigure}{0.49\textwidth}
        \centering
        \includegraphics[width = 0.7\linewidth]{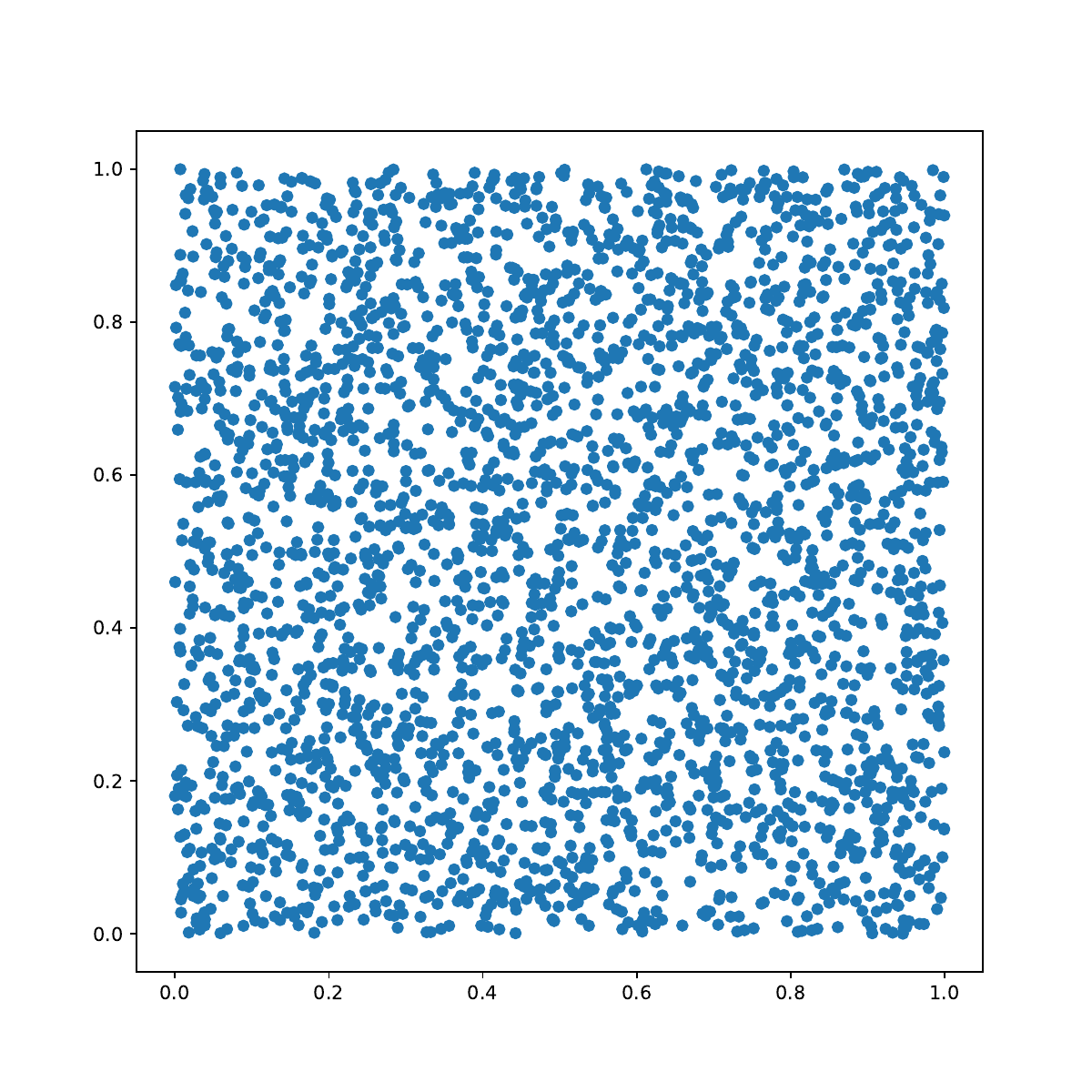}
        \caption{$r=3.5$}
        \label{fig:orbits_3.5}
    \end{subfigure}
    \hfill
    \begin{subfigure}{0.49\textwidth}
        \centering
        \includegraphics[width=0.7\linewidth]{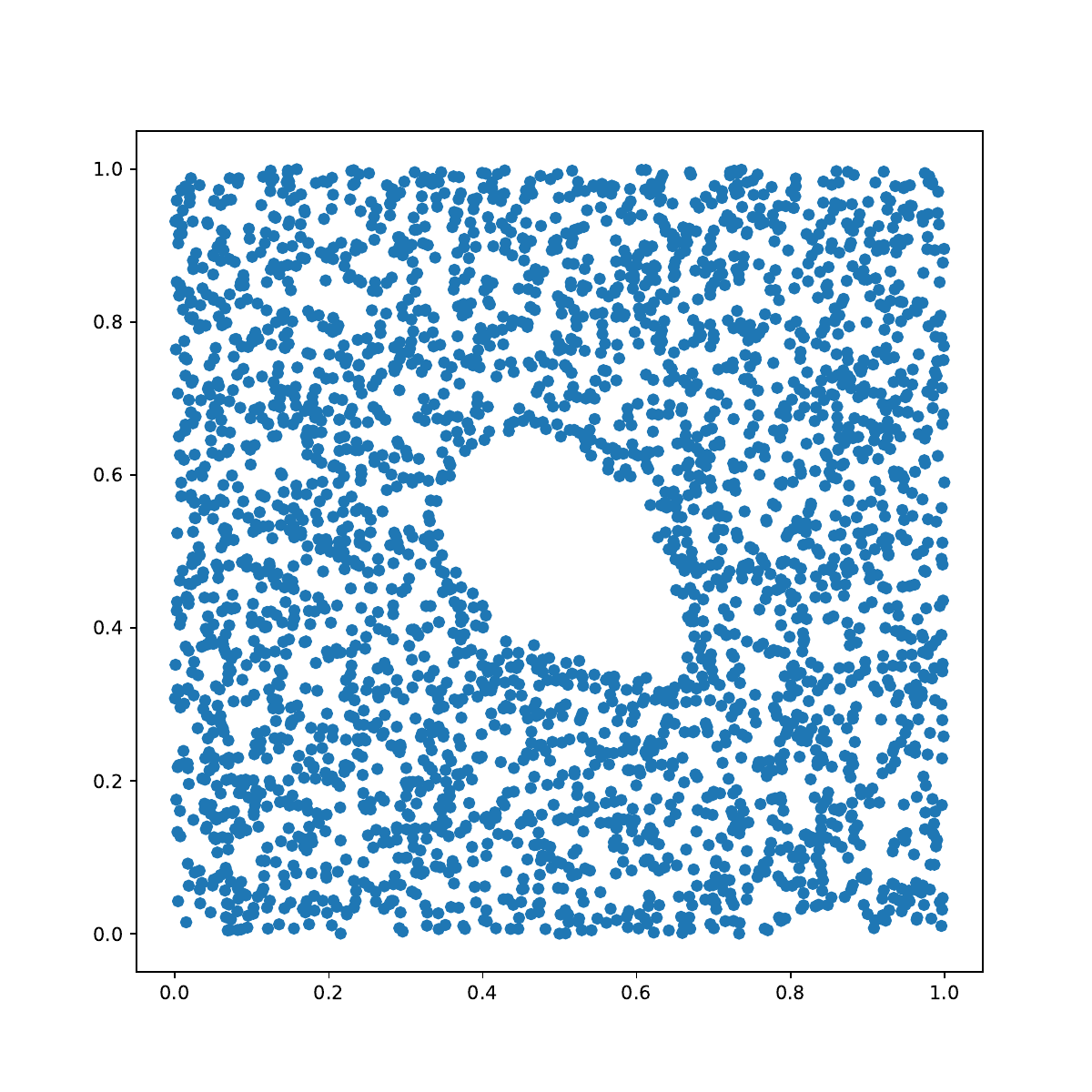}
        \caption{$r=4.1$}
        \label{fig:orbits_4.1}
    \end{subfigure}
    \caption{Two orbits with different parameter values, the classifier is trained to predict the value of the parameter.}
\end{figure}
Following what was done in \citep{pmlr-v70-carriere17a}, we use 5 parameter values $r \in \{2.5, 3.5, 4, 4.1, 4.3\}$ and for each of those we generate 100 orbits with 1000 points and random initial positions. We then use the GUDHI library \citep{gudhi:urm} to compute the persistence diagrams of the distance to the points and use them (in homological dimension 1) to produce an orbit classifier by training over a 60\%-40\% train-test split of the data.

\bmhead*{Results.} As reported in Table \ref{table:results} both $\SFG$ and $\widehat\SFG$ achieve a performance similar to the one of the Sliced Wasserstein Kernel. 
\subsubsection{Texture classification}

\bmhead*{Dataset.} The second experiment uses the \textit{OUTEX} database for texture classification. Since we were unable to access the original \textit{OUTEX0000} dataset we instead use part of an extended version of it: \textit{Outex\textunderscore TC\textunderscore 00010-r} (available at \url{https://color.univ-lille.fr/datasets/extended-outex}). For each of the images in the database we compute the sign component of the CLPB descriptor \citep{clbp} with radius $R = 1$ and neighbours $P = 8$. The output of this descriptor can then be interpreted as a weighted cubical cell complex \citep{reininghaus-pss} from which we again compute a persistence diagram using the GUDHI library. We then use it in homological dimension 0 to produce a classifier over a 50\%-50\% train-test split of the data.

\begin{figure}[!h] \label{fig:textures}
    \centering
    \begin{subfigure}{0.32\textwidth}
        \centering
        \includegraphics[width = 0.7\linewidth]{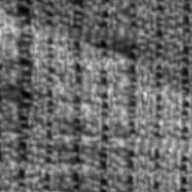}
        \caption{Original texture image}
        \label{fig:texture_png}
    \end{subfigure}
    \hfill
    \begin{subfigure}{0.32\textwidth}
        \centering
        \includegraphics[width=0.7\linewidth]{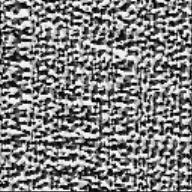}
        \caption{CLBP-S transform}
        \label{fig:texture_lbp}
    \end{subfigure}
    \hfill
    \begin{subfigure}{0.32\textwidth}
        \centering
        \includegraphics[width=0.8\linewidth]{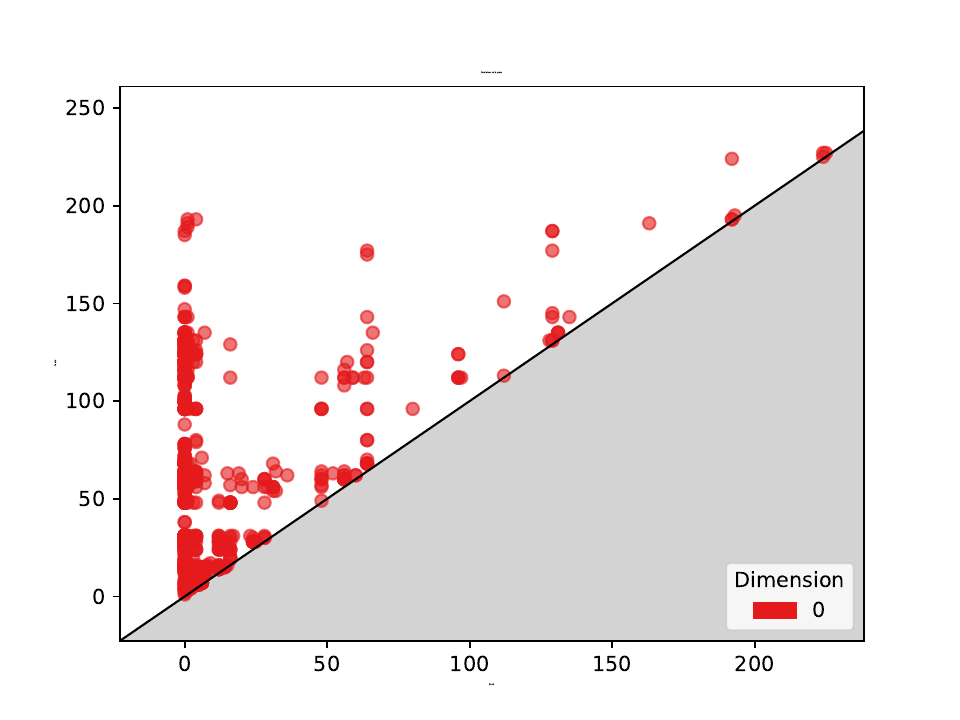}
        \caption{Corresponding persistence diagram}
        \label{fig:texture_pers}
    \end{subfigure}
    \caption{The pipeline we use to produce texture classifiers.}
\end{figure}
\textbf{Results:} As reported in Table \ref{table:results} both $\SFG$ and $\widehat\SFG$ again achieve a performance similar to the one of the Sliced Wasserstein Kernel.

\begin{table}[!h]
\vskip 0.15in
\begin{center}
\begin{small}
\begin{sc}
\begin{tabular}{|l|llll|}
\hline 
Task &         $\SW$ &  $\SFG$ &    $\widehat\SFG$ &         \\
\hline 
Orbit &        $79.7 \pm 1.64$ &       $80.8 \pm 1.2$ &     $79.8 \pm 0.4$ &                     \\        
Texture &      $89.5 \pm 5.2$ &        $90.1 \pm 1.3$ &    $91.7 \pm 2.8$ &                      \\                                           
\hline            
\end{tabular}

\end{sc}
\end{small}
\caption{\label{table:results} Accuracies on the two different experiments we performed: Orbits and Textures.}
\end{center}
\vskip -0.1in
\end{table}

\section{Conclusion}
In this work, we defined a new distance between persistence diagrams that remains faithful to the geometry of the space of PDs while being computationally efficient.
Unlike the Sliced Wasserstein distance, the Sliced Figalli--Gigli distance admits a natural extension to infinite persistence diagrams and to persistence measures, placing it in a broader and more flexible measure-theoretic framework that encompasses existing finite-diagram formulations as special cases.

Left open by our work are several directions for future research. 
On the theoretical side, while we derive in \cref{sec:stability} stability bounds for the $\SFG$ distance, empirical results (see \cref{appendix:complementary-expe}) suggest that the observed stability in practical settings is significantly stronger. Understanding this gap, for instance by identifying refined assumptions on the distributions of persistence diagrams considered under which improved bounds can be established, is an interesting direction for further investigation. 
Following recent advances in the computational optimal transport literature, one may also considered variants of the $\SFG$ distance, e.g.~using a maximum instead of an integration \citep{deshpande2019max,kolouri2019generalized}. 
Another appealing question---going beyond Topological Data Analysis---would be to generalize the $\SFG$ distance to domains more general than the open half-plane $\groundspace$. 
Indeed, in this work, the (somewhat simple) geometry of $\groundspace$ and of the resulting geodesics emanating from $\thediag$ was crucial (and sufficient for our purpose). 
However, the seminal formalism introduced by Figalli and Gigli in \citep{FG10} enable more general domains $(\groundspace, \thediag)$ (e.g.~$\overline{\groundspace}$ being a compact subset of $\bR^d$ and its boundary $\thediag$ being reasonably smooth); defining a sliced Figalli--Gigli distance in that cases would be of interest. 

\clearpage
\bibliography{articles}
\clearpage
\appendix

\section{Sliced Figalli--Gigli distance with a continuous projection} \label{appendix:continuous_proj}

As mentionned in \ref{rem:choice-of-projection}, we can change the projection we use in the definition of the $\SFG_p$ distance. In particular, the projection $z \mapsto \pi_t(z)$ is not continuous which motivates the introduction of \textit{continuous versions} of our constructions. We hoped that this increased regularity would lead to better theoretical guarantees but we did not manage to obtain any. For the sake of completeness, we still present here our results on the continuous $\SFG$ distance.

\begin{definition}[Continuous projections and the continuous SFG distance]
    Let $t \in \bR$ et $z = (z_1, z_2) \in \groundspace$, the \textit{continuous} projection on the geodesic with parameter $t$ is given by:
    \begin{equation} \label{eq:pi_linear}
        \tilde{\pi}_t(z) =
        \begin{cases}
            (t,t) + \frac{z_2-z_1}{2}(1 - \frac{\|z-\pi_t(z)\|}{\|z - \pi(z)\|})(-1,1) & \text{if } z_1 \leq t \leq z_2 \\
            \thediag & \text{otherwise.}
        \end{cases}
    \end{equation}

    The continuous SFG distance is defined for $\mu, \nu \in \cM^p(\groundspace)$ as:
    
    \begin{equation}  \label{eq:def-SFG-continuous} \tag{$\widehat{\SFG}_p$}
    \widehat\SFG_p(\mu,\nu) \coloneqq \frac{p+1}{\sqrt2}\left(\int_\bR \FG_p^p(\tilde\pi_t\# \tilde\mu, \tilde\pi_t \# \tilde\nu)\dd t\right)^{\frac{1}{p}}.
    \end{equation}
\end{definition}

    This distance is also well-defined, following a computation similar to that of \eqref{eq:def-SFG}.

\begin{figure}[!h]
    \centering
    \includegraphics[width=0.49\linewidth]{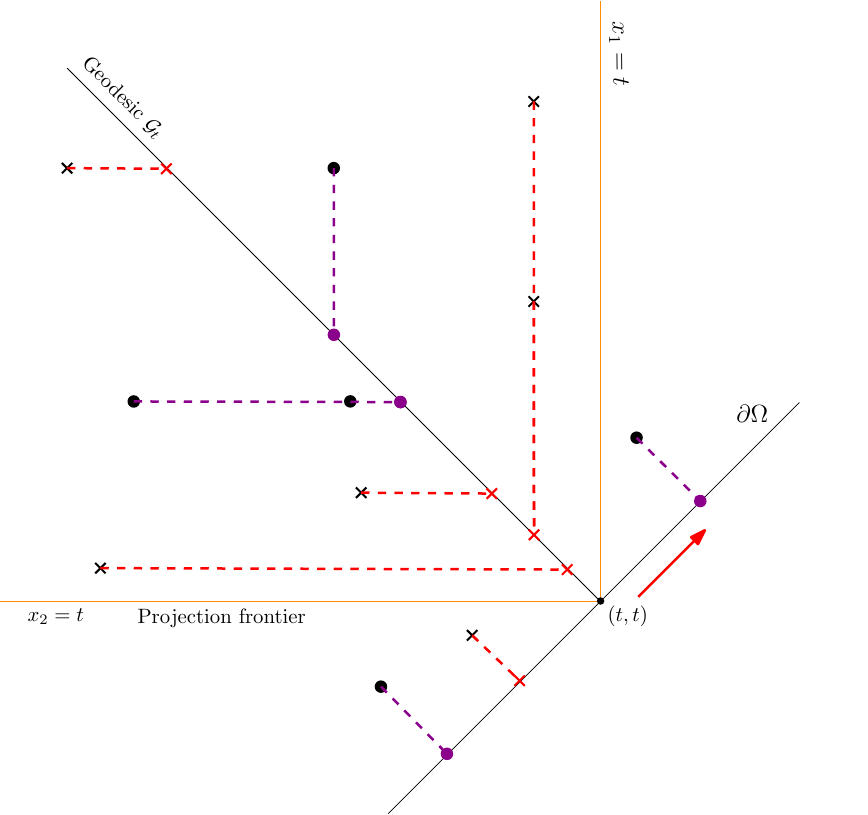}
    \caption{Continuous projection}
    \label{fig:proj-c-SFG}
\end{figure}

\begin{prop} \label{prop:sfgPersContinuous}
    Let $\mu \in \cM^p(\groundspace)$, and $\emptyset$ denote the empty diagram. One has
    \begin{equation} 
    \widehat\SFG_p(\mu, \emptyset) = \FG_p(\mu,\emptyset) = \mathrm{Pers}_p(\mu).
    \end{equation}
\end{prop}

\begin{proof}
The proof is similar to that of \cref{prop:sfg-pers}.
\end{proof}

We now prove that $\widehat\SFG$ is a distance on the space of persistence diagrams $\cD^p(\groundspace)$. 
Extending the result to $\cM^p(\groundspace)$ as we did for $\SFG$---which essentially boils down to prove injectivity of the corresponding transform---would require to adapt 
the proof technique of \cref{prop:SFG-dist}. 
It raises several challenges as the sublevel sets of that projection are not translation invariant. 
We left this possible extension for future work. 
\begin{prop}
    $\widehat\SFG_p$ is a distance on $\cD^p(\groundspace)$
\end{prop}

\begin{proof}
    Let $\mu, \nu \in \cD^p(\groundspace)$ such that $\mu \neq \nu$ and define :
    $$A = \{x \in \supp(\mu) \cup \supp(\nu) : d(x, \thediag) = \max(\mathrm{Pers}_\infty(\mu), \mathrm{Pers}_\infty(\nu)) \text{ and } \mu(x) \neq \nu(x)\}$$
    Observe that $A$ is finite since we consider diagrams with finite persistence. Let $x \in A$ be such that $x_2$ is maximal and assume without loss of generality that $x \in \supp(\mu)$ and $\mu(x) > \nu(x)$. Then, there exists $\varepsilon > 0$ such that for all $t \in I := [\frac{x_1 + x_2}{2} - \varepsilon, \frac{x_1 + x_2}{2} + \varepsilon[, \pi_t\#\tilde\mu(x) > \pi_t\#\tilde\nu(x)$ and therefore $\FG_p(\tilde\pi_t \# \tilde\mu, \tilde\pi_t\#\tilde\nu) > 0$. Since $|I| >0$, we also deduce that $\widehat\SFG_p(\mu,\nu) \neq 0$.\\
\end{proof}

\textbf{Stability of $\widehat\SFG$}. Similarly to what was done for $\SFG_p$ we prove that $\widehat\SFG_p$ is sable with respect to $\FG_p$. We obtain the following:
\begin{prop} \label{prop:ineg-projc}
    Let $\mu, \nu \in \cM^{p}(\groundspace)\cap\cM^\infty(\groundspace)$ one has:
    \begin{equation} 
    \widehat\SFG_p^p(\mu,\nu) \leq \sqrt2^p(p +2)\FG_p^p(\mu,\nu) + M^{p-1}\iint_{\Omega\times \Omega} d(x,y)\dd \gamma(x,y),
    \end{equation}
    where $M = \max(\mathrm{Pers}_\infty(\mu), \mathrm{Pers}_\infty(\nu))$.
\end{prop}

\begin{rem} \label{rem:ineg-projn-1}
    In the case of $p=1$, it is not necessary to suppose that $\mu$ and $\nu$ are in $\cM^\infty(\groundspace)$ and the above inequality becomes:
\begin{equation}
    \forall\mu, \nu \in \cM^1(\groundspace), \:\widehat{\SFG}_1(\mu,\nu) \leq (3\sqrt2 + 1)\FG_1(\mu,\nu).
\end{equation}
\end{rem}

Before proving \ref{prop:ineg-projc}, we first need a few straightforward results about the projections $\tilde\pi_t$. Let $x, y \in \groundspace$ and $t \in \bR$, we distinguish two cases:
If $x_1 \leq t \leq x_2$ and $y_1 \leq t \leq y_2$:
\begin{align}
    \|\tilde\pi_t(x) - \tilde\pi_t(y)\| &= 
    \begin{cases}
    \sqrt{2}|y_1 - x_1| &\text{if } t\leq \frac{x_1+x_2}{2} \text{ and } t \leq \frac{y_1 + y_2}{2}\\
    \sqrt2|y_2 - x_2| &\text{if } t \geq \frac{x_1 + x_2}{2} \text{ and } t \geq \frac{y_1 + y_2}{2}\\
    \sqrt{2}|y_1 + x_2 - 2t| &\text{if } \frac{x_1 + x_2}{2} \leq t \leq \frac{y_1 + y_2}{2}\\
    \sqrt2|y_2 + x_1 - 2t|&\text{if } \frac{y_1 + y_2}{2} \leq t \leq \frac{x_1 + x_2}{2} 
    \end{cases} \\
    &\leq \sqrt2\|x - y\|
\end{align}
If $y_1 \leq t \leq y_2$ and $t \notin [x_1, x_2]$: 
\begin{equation}
    \|\tilde\pi_t(x) - \tilde\pi_t(y)\| =
    \begin{cases}
        \sqrt2(t-y_1) &\text{if } t \leq \frac{y_1+y_2}{2} \\
        \sqrt2(y_2 -t) &\text{if } t \geq \frac{y_1 + y_2}{2}
    \end{cases}
\end{equation}
From this, we deduce:
\begin{lemma} \label{lemma-int-projc}
    \begin{equation} 
    \int_\bR d(\tilde\pi_t(x), \tilde\pi_t(y))^p\dd t \leq \sqrt2^p\left(\frac{\sqrt2}{p+1} + \sqrt2\right)\max(d(x, \thediag), d(y, \thediag))\|x-y\|^p.
    \end{equation}
\end{lemma}
\begin{proof}
Suppose that $x_1 \leq y_1 \leq x_2 \leq y_2$, we then get:
\begin{align}
    \int_\bR d(\tilde\pi_t(x), \tilde\pi_t(y))^p\dd t =\int_{x_1}^{y_1}\|\tilde\pi_t(x) - \pi(\tilde\pi_t(x))\|^p\dd t &+
    \int_{y_1}^{x_2}\|\tilde\pi_t(x) - \tilde\pi_t(y)\|^p\dd t \\ &+
    \int_{x_2}^{y_2}\|\tilde\pi_t(y) - \pi(\tilde\pi_t(y)\|^p\dd t.
\end{align}
Using the results above, if $y_1 \leq \frac{x_1+x_2}{2}$ we have
\begin{align} \int_{x_1}^{y_1}\|\tilde\pi_t(x) - \pi(\tilde\pi_t(x))\|^p\dd t &= \int_{x_1}^{y_1}[\sqrt2(t-x_1)]^p\dd t \\&= \frac{\sqrt{2}^p}{p+1}(y_1 - x_1)^{p+1} \leq \frac{\sqrt{2}^p}{p+1}(x_2-x_1)(y_1-x_1)^p,
\end{align}
and if $y_1 \geq \frac{x_1 + x_2}{2}$ we have
\begin{align} 
\int_{x_1}^{y_1}\|\tilde\pi_t(x) - \pi(\tilde\pi_t(x))\|^p\dd t &=
\frac{\sqrt{2}^p}{p+1}(\frac{x_2-x_1}{2})^{p+1} + \int_{\frac{x_1+x_2}{2}}^{y_1}[\sqrt2(x_2-t)]^p\dd t \\&\leq \frac{\sqrt{2}^p}{p+1}(x_2-x_1)(y_1-x_1)^p.
\end{align}
Similarly,
\begin{equation} 
\int_{x_2}^{y_2}\|\tilde\pi_t(y) - \pi(\tilde\pi_t(y))\|^p\dd t \leq \frac{\sqrt2^p}{p + 1}(y_2 - y_1)(y_2 - x_2)^p.
\end{equation}
Furthermore, using the results above we get:
$\int_{y_1}^{x_2}\|\tilde\pi_t(x) - \tilde\pi_t(y)\|^p\dd t \leq \sqrt2^p(x_2 - x_1)\|x-y\|^p$ from which we deduce the desired result.
The calculation is the same when $x_1 \leq y_1 \leq y_2 \leq x_2$.
\end{proof}

Then, let $\gamma$ be the optimal transport plan achieving the infimum in $\FG_p(\mu,\nu)$. We can then define $\tilde\gamma \in \Adm(\tilde\mu, \tilde\nu)$ using \cref{prop:tilde-transport-plan}, yielding he following result.
\begin{prop} \label{prop:ineg-projc-groundspace}
It holds that
    \begin{equation}
    \int_\bR\iint_{\Omega\times\Omega} d(\tilde\pi_t(x), \tilde\pi_t(y))^p\dd \tilde\gamma(x,y) \dd t \leq \sqrt{2}^{p+1}\left(1 + \frac{1}{p+1}\right)\iint_{\groundspace\times\groundspace}\|x - y\|^p\dd \gamma(x,y).
\end{equation}
\end{prop}

\begin{proof}
    Observe that $\supp(\gamma) \cap \{(x,y) \in \Omega^2 : x_1 \leq x_2 \leq y_1 \leq y_2 \text{ or } y_1 \leq y_2 \leq x_1 \leq x_2\} = \emptyset$, the result is then a consequence of \cref{lemma-int-projc} and of the definition of $\tilde\gamma$.
\end{proof}

We can then write:
\begin{align}
    \int_\bR\iint_{\groundspace\times\thediag}d(\tilde\pi_t(x), &\tilde\pi_t(y))\dd \tilde\gamma(x,y)\dd t = \frac{\sqrt2}{p+1}\iint_{\groundspace\times\thediag}d(x,\thediag)^{p+1}\frac{\dd\gamma(x,y)}{d(x, \thediag)} \\
    &\quad\quad+ \frac{\sqrt2}{p+1}\iint_{\groundspace_{1,+}}d(x, \thediag)^{p+1}\left(\frac{1}{d(x,\thediag)} - \frac{1}{d(y,\thediag)} \right)\dd \gamma(x,y)
\end{align}
And:
\begin{align}
    \int_\bR\iint_{\thediag\times\groundspace}d(\tilde\pi_t(x), &\tilde\pi_t(y))\dd \tilde\gamma(x,y)\dd t = \frac{\sqrt2}{p+1}\iint_{\thediag\times\groundspace}d(y,\thediag)^{p+1}\frac{\dd\gamma(x,y)}{d(y, \thediag)} \\
    &\quad\quad+ \frac{\sqrt2}{p+1}\iint_{\groundspace_{2,+}}d(y, \thediag)^{p+1}\left(\frac{1}{d(y,\thediag)} - \frac{1}{d(x,\thediag)} \right)\dd \gamma(x,y)
\end{align}
Where $\groundspace_{1,+} = \{(x, y) \in \groundspace\times\groundspace: d(y,\thediag) \geq d(x, \thediag)\}$ and $\groundspace_{2,+} = \{(x, y) \in \groundspace\times\groundspace: d(x,\thediag) \geq d(y, \thediag)\}$. Combining this with \cref{prop:ineg-projc-groundspace} we get:
\begin{align} \label{eq:ineg-projc-sharp}
    \begin{split}
    \frac{\sqrt2}{p+1}\widehat{\SFG}_p^p(\mu, \nu) \leq &\frac{\sqrt{2}^{p+1}}{p+1}\FG_p^p(\mu,\nu) + \sqrt{2}^{p+1}\iint_{\groundspace\times\groundspace}d(x,y)^p\dd \gamma(x,y) \\ &+ \frac{\sqrt2}{p+1}\iint_{\groundspace\times\groundspace}\frac{\min(d(x,\thediag), d(y,\thediag))^{p}}{\max(d(x,\thediag),d(y,\thediag))}d(x,y)\dd\gamma(x,y)
    \end{split}
\end{align}
And, if we further suppose that $\mu, \nu \in \cM^{\infty}(\Omega)$ we finally obtain the result of \ref{prop:ineg-projc}. \\

As in the case of $\SFG$, a uniform lower bound is not obtainable for $\widehat\SFG$. This time, the diagrams to consider are $\mu = \bigcup_{n=1}^p\{(-\alpha + 2n\delta, \alpha), (-\alpha + (2n +1)\delta, \alpha + \delta\}$ and $\nu = \bigcup_{n=1}^p\{(-\alpha + 2n\delta, \alpha + \delta), (-\alpha + (2n +1)\delta, \alpha\}$. One can then show that $\widehat{\SFG}(\mu, \nu)$ is on the order of $\delta^{p+1}$ when $\FG(\mu,\nu)$ is on the order of $\delta^p$. \\

\textbf{Topological equivalence of $\widehat\SFG$ and $\FG$}. We proceed in the same way as for $\SFG$ to show that convergence for $\FG_p$ implies convergence for $\widehat\SFG_p$.
\begin{prop}
    Let $\mu \in \cD^p(\groundspace)$ and $(\mu_n)$ be a sequence of $\cD^p(\groundspace)$ such that $\mu_n \xrightarrow[]{\FG_p} \mu$ then $\mu_n \xrightarrow[]{\widehat\SW_p} 0$.
\end{prop}

To prove the converse implication we will again show that convergence for $\widehat\SFG$ implies characterisation (\ref{eq:convergence}). Like previously, we easily get the convergence of persistences:
\begin{prop}
    Let $\mu \in \cD^p(\groundspace)$ and $(\mu_n)$ be a sequence of $\cD^p(\groundspace)$ such that $\mu_n \xrightarrow[]{\SW_p} \mu$. Then $\mathrm{Pers}_p(\mu_n) \to \mathrm{Pers}_p(\mu)$.
\end{prop}
\begin{proof}
    We have $\widehat\SW_p(\nu, \emptyset) = \mathrm{Pers}_p(\nu)$ for all $\nu \in \cM^p(\groundspace)$. The result then follows from the triangle inequality.
\end{proof}
In order to prove that the $\widehat\SFG_p$ convergence implies vague convergence we will proceed in a similar manner as for $\SFG$ by simply changing the shape of the region $B_\varepsilon$. 

\begin{definition}
    We now denote by $\hat B_{\varepsilon, t}$ the "elbow" of width $\varepsilon$ centered at $x + (t-t_x)(1,1)$ where $t_x = \frac{x_1 + x_2}{2}$ (see Figure 3). We denote by $\hat B_{\varepsilon, v, t}$ (resp. $\hat B_{\varepsilon, h, t}$) the vertical (resp. horizontal) part of $\hat B_{\varepsilon,t}$. We also define the following quantity (similar to $\tilde\SFG$):
    \begin{equation} 
    \bar\SFG_{\varepsilon, \eta}(\mu,\nu) \coloneqq \frac{1}{\sqrt2}\left(\int_\bR\tilde\FG_{B_{\varepsilon,t},B_{\eta, t}}^p(\tilde\pi_t\#\tilde\mu,\tilde\pi_t\#\tilde\nu)\dd t\right)^{\frac{1}{p}}.
    \end{equation}
    
    As before, we also define $\Delta_{t, \varepsilon} \coloneqq \tilde\mu_n(B_{\varepsilon,t}) - \tilde\mu(B_{\varepsilon,t})$, $\Delta_{v,t, \varepsilon} =\tilde\mu_n(B_{\varepsilon,v,t}) - \tilde\mu(B_{\varepsilon,v,t})$ and $\Delta_{h,t, \varepsilon} = \tilde\mu_n(B_{\varepsilon,h,t}) - \tilde\mu(B_{\varepsilon,h,t})$.
\end{definition}

\begin{figure}[!h] \label{fig:def-B-continuous}
    \centering
    \begin{subfigure}{0.32\textwidth}
        \centering
        \includegraphics[width = 0.92\linewidth]{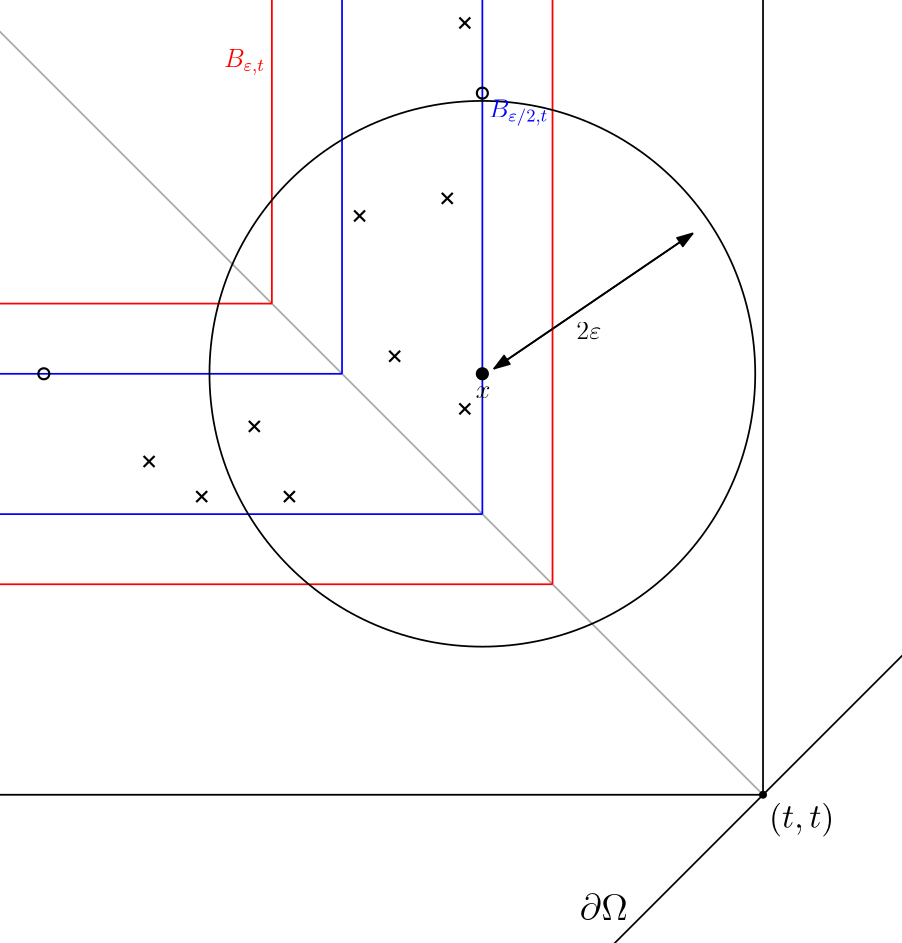}
        \caption{$B_{\varepsilon/2,t}$ at $t_0$}
        \label{fig:projc_t0}
    \end{subfigure}
    \hfill
    \begin{subfigure}{0.32\textwidth}
        \centering
        \includegraphics[width=0.9\linewidth]{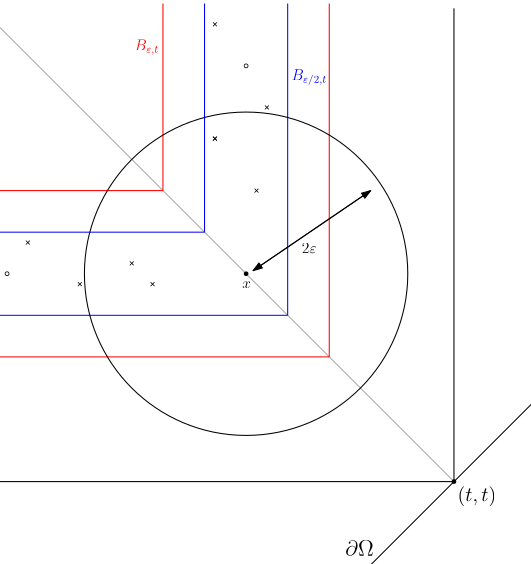}
        \caption{$B_{\varepsilon/2,t}$ at $t_x$}
        \label{fig:projc_tx}
    \end{subfigure}
    \hfill
    \begin{subfigure}{0.32\textwidth}
        \centering
        \includegraphics[width=0.9\linewidth]{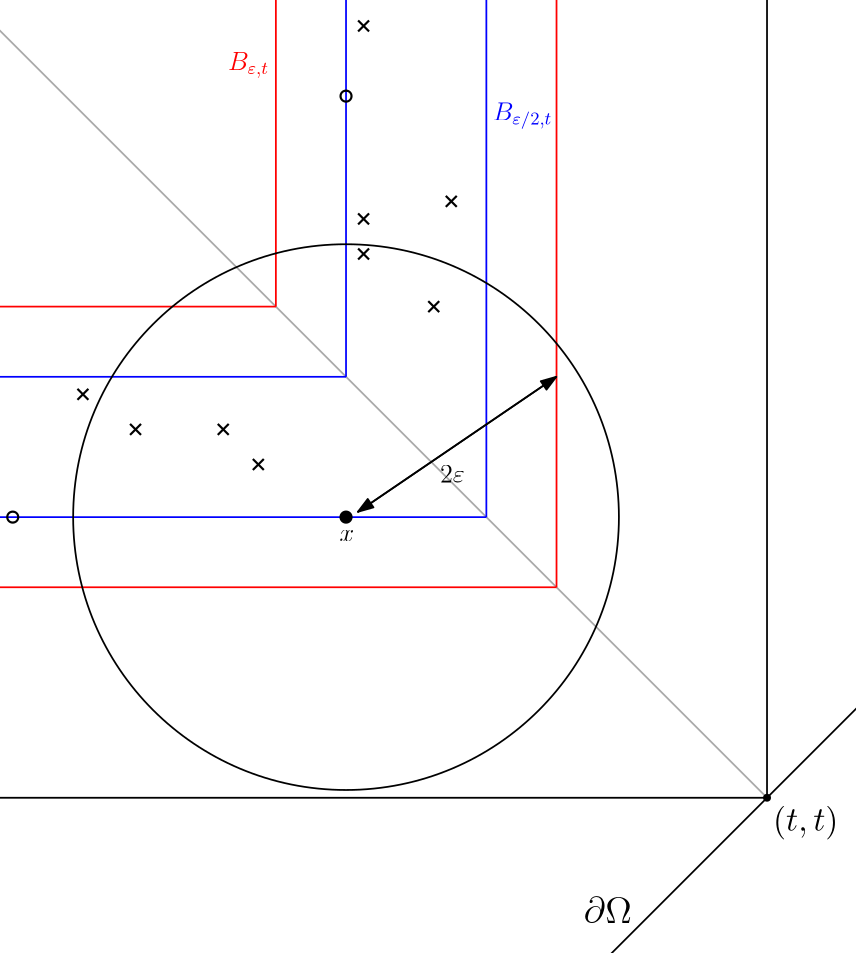}
        \caption{$B_{\varepsilon/2,t}$ at $t_1$}
        \label{fig:projx_t1}
    \end{subfigure}
    \caption{The construction we use for the proof of Proposition \ref{prop:top-SFG}. $B_{\varepsilon,t}$ is in red, $B_{\varepsilon/2,t}$ is in blue. Points of $\mu_n$ are represented by crosses, and those of $\mu$ by circles. $\varepsilon$ was chosen small enough so that all circles lie on $D_{x,v} \cup D_{x,h}$}
\end{figure}

 With these new definitions, the results of Lemmas \ref{lem:trop-projn} and \ref{lem:technical-top-n} still hold. The proof is then very similar to that of Proposition \ref{prop:top-SFG}, we argue by contradiction to lower bound $\bar\SFG$ by studying $\hat\Delta_t, \hat\Delta_{v,t}$ and $\hat\Delta_{h,t}$.

\begin{prop}
    Let $\mu \in \cD^p(\groundspace)$ and $(\mu_n)$ a sequence of $\cD^p(\groundspace)$ satisfying $\mu_n \xrightarrow[]{\widehat\SW_p}\mu$. Then, for all $x \in \supp(\mu)$, there exists $\eta > 0$ such that for any $0 < \varepsilon < \eta$, $\mu_n(1_{B(x, \varepsilon)}) \to \mu(x)$
\end{prop}

\begin{proof}
We again argue by contradiction and assume there exists a point $x$ such that for all $\eta > 0$, there exists $\varepsilon > 0$ verifying $\mu_n(1_{B(x, 2\varepsilon)}) \neq \mu(x)$ for infinitely many $n$. We denote by $t_0$ the instant at which $x$ appears in $B_{\varepsilon/2, t}$ and $t_1$ the instant at which it disappears. We also set $B_{\varepsilon,x} \coloneqq\bigcup_{t\in [t_0,t_1]} B_{\varepsilon,t}$. Similarly to what we did before, we can reduce $\varepsilon$ so that the only points of $\supp(\mu) \cap B_{\varepsilon,x}$ are contained in $D_{x,v} = \{y \in \groundspace : y_1 = x_1 \text{ and } y_2 > x_2\}$ or $D_{x,h} = \{y \in \groundspace : y_2 = x_2\text{ and } y_1 < x_1 \}$. We will now lower bound $\bar\SFG_{\varepsilon/2, \varepsilon}(\mu_n,\mu)$ to obtain a contradiction.\\ \\
First, there cannot be $y \in \supp(\mu_n)$ verifying $d(y,\thediag) > \mathrm{Pers}_\infty(\mu) + 1$ or we would have $\widehat\SFG(\mu_n,\mu) > \frac{C}{\mathrm{Pers}_\infty(\mu) + 1}$. Lastly any point appearing in $B_{\varepsilon/2, t}$ between $t_0$ and $t_1$ cannot disappear in that same time frame. As such, we deduce that at most $K \coloneqq (\mathrm{Pers}_\infty(\mu) + 1)(\mu(B_{\varepsilon}) + 1)$ points of $\mu_n$ appear in $B_{\varepsilon/2, t}$ between $t_0$ et $t_1$.

We will now show that we cannot have $\Delta_{v,t} > \varepsilon/8$ or $\Delta_{h,t} > \varepsilon/8$ for a duration longer than $\varepsilon/8$. For this, we need the following lemma
\begin{lemma} \label{lem:min-delta-vh}
    Let $t_1 < t_2 \in \bR, x \in B_{\varepsilon,h, t_1}$ and  $y\in B_{\varepsilon,v, t_2}$ then:
    \begin{equation} 
    \int_{t_1}^{t_2}d(\tilde\pi_t(x), \tilde\pi_t(y))\dd t \geq (t_2 - t_1)^2.
    \end{equation}
\end{lemma}

\begin{proof}
    We denote $x = (x_1, x_2)$ and $y= (y_1, y_2)$. We have
    \begin{align}
        \int_{t_1}^{t_2}d(\tilde\pi_t(x), \tilde\pi_t(y))\dd t &= 2\int_{t_1}^{t_2}|x_1 + y_2 - 2t|\dd t \\
        & \geq 2\left[(t_1 - \frac{x_1 + y_2}{2})^2 + (t_2 - \frac{x_1 + y_2}{2})^2\right] \\
        &\geq (t_2 -t_1)^2
    \end{align}
\end{proof}

Suppose now for a contradiction that there exists $\alpha \in \bR$ such that $\Delta_{v,t} > \varepsilon/8$ or $\Delta_{h,t} > \varepsilon/8$ for all $t \in [\alpha, \alpha+\varepsilon/8]$. Observe that the optimal transport plan $\pi_{n,t}$ from $\pi_t\#\tilde\mu_n$ to $\pi_t\#\tilde\mu$ only changes in finitely many instants $t \in [\alpha, \alpha+\varepsilon/8]$: those where a point of $\mu_n$ and a point of $\mu$ have the same projection under $\pi_t$ and those where a point of $\mu_n$ disappears/appears in $B_{\varepsilon/2,t}$. This can happen at most $4K$ times. Thus, there exists $\alpha = t_1 \leq \dots \leq t_{4K} = \alpha + \varepsilon/8$ such that for all $1 \leq i \leq 4K$ and all $t \in [t_i, t_{i+1}[$ we have $\pi_{n,t} = \pi_{n,t_i}$. Furthermore, since we are supposing that $\Delta_{v,t} > \varepsilon/8$ or $\Delta_{h,t} > \varepsilon/8$, we get the existence of $A_t \subset B_{\varepsilon/2,t,h} \times B_{\varepsilon/2,t,v} \cup B_{\varepsilon/2,t,v} \times B_{\varepsilon/2,t,h}$ and $C_t \subset B_{\varepsilon/2, t} \times \partial B_{\varepsilon,t} \cup \partial B_{\varepsilon/2,t} \times (D_{x,v}\cup D_{x,h})$ such that $\pi_{n,t}(A_t \cup C_t) > \varepsilon/8$ for all $t \in [\alpha, \alpha + \varepsilon/8]$. We can then write:
\begin{align}
    \int_\alpha^{\alpha + \varepsilon/8}\iint_{\overline{B_{\varepsilon/2,t}}\times \overline{B_{\varepsilon/2,t}}}d(\tilde\pi_t(x), &\tilde\pi_t(y))\dd \pi_{n,t}(x,y) \dd t \\&= 
    \sum_{i=1}^{4K}\int_{t_i}^{t_{i+1}} \iint_{\overline{B_{\varepsilon/2,t}}\times \overline{B_{\varepsilon/2,t}}}d(\tilde\pi_t(x), \tilde\pi_t(y))\dd \pi_{n,t_i}(x,y) \dd t \\
    &\geq \sum_{i=1}^{4K}\iint_{A_{t_i}\cup C_{t_i}}\int_{t_i}^{t_{i+1}}d(\tilde\pi_t(x), \tilde\pi_t(y))\dd t \dd \pi_{n,t_i}(x,y) \\
    &\geq \sum_{i=1}^{4K}\pi_{n,t_i}(A_{t_i})(t_{i+1} - t_i)^2 + \pi_{n,t_i}(C_{t_i})\times\varepsilon/2\tag{according to \cref{lem:min-delta-vh}} \\
    & \ge \frac{\varepsilon}{8}\times \frac{(\varepsilon/8)^2}{4K} \tag{by possibly reducing $\varepsilon$ below 1}
\end{align}

As such we cannot have $\Delta_{v,t} > \varepsilon/8$ or $\Delta_{h,t} > \varepsilon/8$ for longer than $\varepsilon/8$. However at the instant $t_x = \frac{x_1 + x_2}{2}$, the point $x$ of multiplicity $\mu(x)$ goes from $B_{\varepsilon/2,v}$ to $B_{\varepsilon/2,h}$ whereas between $t_x - \varepsilon/2$ and $t_x + \varepsilon/2$, $r \neq \mu(x)$ points of $\mu_n$ go from $B_{\varepsilon/2,v}$ to $B_{\varepsilon/2,h}$. As such, around $t_x$, at least one point of $\mu_n$ must disappear in $B_{\varepsilon/2, v}$ and appear in $B_{\varepsilon/2,h}$ (or the other way around). But then, this would imply that $\Delta_{3\varepsilon/4,t} > \frac{3}{4(\mathrm{Pers}_\infty(\mu) + 1)}$ for a duration at least $\varepsilon/4$, and thus we would get $\bar\SFG \geq \frac{3\varepsilon^2}{32(\mathrm{Pers}_\infty(\mu)+1)}$ which yields the desired contradiction.
\end{proof}

\end{document}